\documentclass[sigconf]{aamas} 

\usepackage{balance} 

\setcopyright{ifaamas}

\copyrightyear{2024}
\acmYear{2024}
\acmDOI{}
\acmPrice{}
\acmISBN{}

\acmSubmissionID{470}

\DeclareTextSymbolDefault{\textquotedbl}{T1}
\providecommand{\tabularnewline}{\\}
\floatstyle{ruled}
\newfloat{algorithm}{tbp}{loa}
\providecommand{\algorithmname}{Algorithm}
\floatname{algorithm}{\protect\algorithmname}

\theoremstyle{plain}
\newtheorem{thm}{\protect\theoremname}
\theoremstyle{plain}
\newtheorem{prop}[thm]{\protect\propositionname}
\usepackage{array}
\usepackage{float}
\usepackage{url}
\usepackage{multirow}

\usepackage{graphicx}
\usepackage{url}
\usepackage{hyperref}
\usepackage{algorithm}
\usepackage{algorithmic}
\usepackage{microtype}
\renewcommand{\cite}{\citep}

\makeatother

\usepackage{babel}
\providecommand{\propositionname}{Proposition}
\providecommand{\theoremname}{Theorem}

\title{Beyond Surprise: Improving Exploration Through Surprise Novelty}
\author{Hung Le}
\affiliation{
  \institution{Applied AI Institute, Deakin University}
  \city{Geelong}
  \country{Australia}}
\email{thai.le@deakin.edu.au}

\author{Kien Do}
\affiliation{
  \institution{Applied AI Institute, Deakin University}
  \city{Geelong}
  \country{Australia}}
\email{k.do@deakin.edu.au}

\author{Dung Nguyen}
\affiliation{
  \institution{Applied AI Institute, Deakin University}
  \city{Geelong}
  \country{Australia}}
\email{dung.nguyen@deakin.edu.au}

\author{Svetha Venkatesh}
\affiliation{
  \institution{Applied AI Institute, Deakin University}
  \city{Geelong}
  \country{Australia}}
\email{svetha.venkatesh@deakin.edu.au}

\begin{abstract}
We present a new computing model for intrinsic rewards in reinforcement
learning that addresses the limitations of existing surprise-driven
explorations. The reward is the\emph{ novelty of the surprise} rather
than the surprise norm. We estimate the surprise novelty as retrieval
errors of a memory network wherein the memory stores and reconstructs
surprises. Our surprise memory (SM) augments the capability of surprise-based
intrinsic motivators, maintaining the agent's interest in exciting
exploration while reducing unwanted attraction to unpredictable or
noisy observations. Our experiments demonstrate that the SM combined
with various surprise predictors exhibits efficient exploring behaviors
and significantly boosts the final performance in sparse reward environments,
including Noisy-TV, navigation and challenging Atari games. 
\end{abstract}
\makeatletter
\gdef\@copyrightpermission{
	\begin{minipage}{0.3\columnwidth}
		\href{https://creativecommons.org/licenses/by/4.0/}{\includegraphics[width=0.90\textwidth]{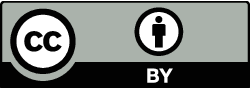}}
	\end{minipage}\hfill
	\begin{minipage}{0.7\columnwidth}
		\href{https://creativecommons.org/licenses/by/4.0/}{This work is licensed under a Creative Commons Attribution International 4.0 License.}
	\end{minipage}
	\vspace{5pt}
}
\makeatother
\begin{document}\sloppy
\pagestyle{fancy}
\fancyhead{}
\maketitle

\section{Introduction}

\emph{What motivates agents to explore?} Successfully answering this
question would enable agents to learn efficiently in formidable tasks.
Random explorations such as $\epsilon$-greedy are inefficient in
high dimensional cases, failing to learn despite training for hundreds
of million steps in sparse reward games \cite{bellemare2016unifying}.
Alternative approaches propose to use intrinsic motivation to aid
exploration by adding bonuses to the environment's rewards \cite{bellemare2016unifying,stadie2015incentivizing}.
The intrinsic reward is often proportional to the novelty of the visiting
state: it is high if the state is novel (e.g. different from the past
ones \cite{badia2020agent57,badia2019never}) or less frequently visited
\cite{bellemare2016unifying,tang2017exploration}. 

Another view of intrinsic motivation is from surprise, which refers
to the result of the experience being unexpected, and is determined
by the discrepancy between the expectation (from the agent\textquoteright s
prediction) and observed reality \cite{barto2013novelty,schmidhuber2010formal}.
Technically, surprise is the difference between prediction and observation
representation vectors. The norm of the residual (i.e. prediction
error) is used as the intrinsic reward. Here, we will use the terms
``surprise'' and ``surprise norm'' to refer to the residual vector
and its norm, respectively.  Recent works have estimated surprise
with various predictive models such as dynamics \cite{stadie2015incentivizing},
episodic reachability \cite{savinov2018episodic} and inverse dynamics
\cite{pathak2017curiosity}; and achieved significant improvements
with surprise norm \cite{burda2018large}. However, surprise-based
agents tend to be overly curious about noisy or unpredictable observations
\cite{itti2005bayesian,schmidhuber1991curious}. For example, consider
an agent watching a television screen showing white noise (noisy-TV
problem). The TV is boring, yet the agent cannot predict the screen's
content and will be attracted to the TV due to its high surprise norm.
This distraction or \textquotedbl fake surprise\textquotedbl{} is
common in partially observable Markov Decision Process (POMDP), including
navigation tasks and Atari games \cite{burda2018exploration}. Many
works have addressed this issue by relying on the learning progress
\cite{achiam2017surprise,schmidhuber1991curious} or random network
distillation (RND) \cite{burda2018exploration}. However, the former
is computationally expensive, and the latter requires many samples
to perform well. 

\begin{figure*}[t]
\begin{centering}
\includegraphics[width=0.8\linewidth]{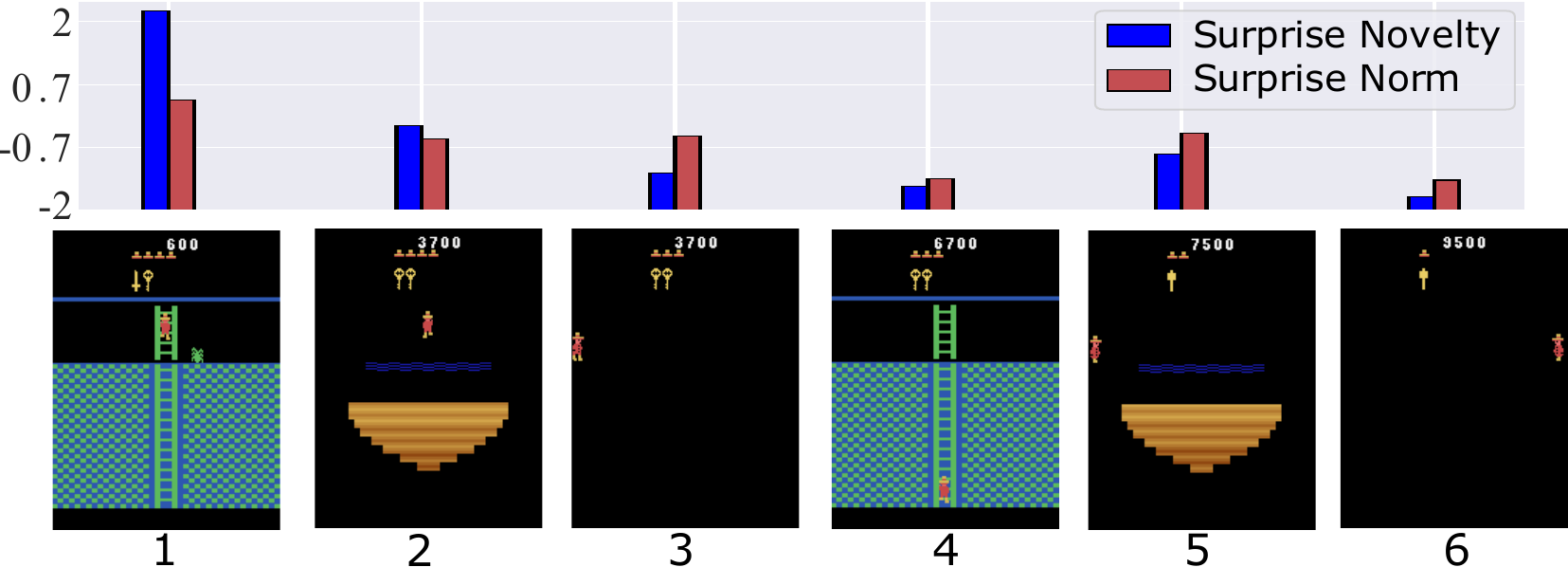}
\par\end{centering}
\caption{Montezuma Revenge: \emph{surprise novelty }better reflects the originality
of the environment than \emph{surprise norm}. While surprise norm
can be significant even for dull events such as those in the dark
room due to unpredictability, surprise novelty tends to be less ($3^{rd}$
and $6^{th}$ image). On the other hand, surprise novelty can be higher
in truly vivid states on the first visit to the ladder and island
rooms ($1^{st}$ and $2^{nd}$ image) and reduced on the second visit
($4^{th}$ and $5^{th}$ image). Here, surprise novelty and surprise
norm are quantified and averaged over steps in each room. \label{fig:Montezuma-Revenge:-surprise}}
\end{figure*}
This paper overcomes the \textquotedbl fake surprise\textquotedbl{}
issue by using \emph{surprise novelty} - a new concept that measures
the uniqueness of surprise. To identify surprise novelty, the agent
needs to compare the current surprise with surprises in past encounters.
One way to do this is to equip the agent with some kind of associative
memory, which we implement as an autoencoder whose task is to reconstruct
a query surprise. The lower the reconstruction error, the lower the
surprise novelty. A further mechanism is needed to deal with the rapid
changes in surprise structure within an episode. As an example, if
the agent meets the same surprise at two time steps, its surprise
novelty should decline, and with a simple autoencoder this will not
happen. To remedy this, we add an episodic memory, which stores intra-episode
surprises. Given the current surprise, this memory can retrieve similar
\textquotedblleft surprises\textquotedblright{} presented earlier
in the episode through an attention mechanism. These surprises act
as a context added to the query to help the autoencoder better recognize
whether the query surprise has been encountered in the episode or
not. The error between the query and the autoencoder's output is defined
as \emph{surprise novelty}, to which the intrinsic reward is set proportionally. 

We argue that using surprise novelty as an intrinsic reward is better
than surprise norm. As in POMDPs, surprise norms can be very large
since the agent cannot predict its environment perfectly, yet there
may exist patterns of prediction failure. If the agent can remember
these patterns, it will not feel surprised when similar prediction
errors appear regardless of the surprise norms. An important emergent
property of this architecture is that when random observations are
presented (e.g., white noise in the noisy-TV problem), the autoencoder
can act as an identity transformation operator, thus effectively \textquotedblleft passing
the noise through\textquotedblright{} to reconstruct it with low error.
We conjecture that the autoencoder is able to do this with the surprise
rather than the observation as the surprise space has lower variance,
and we show this in our paper. Since our memory system works on the
surprise level, we need to adopt current intrinsic motivation methods
to generate surprises. The surprise generator (SG) can be of any kind
based on predictive models mentioned earlier and is jointly trained
with the memory to optimize its own loss function. To train the surprise
memory (SM), we optimize the memory's parameters to minimize the reconstruction
error.

Our contribution is two-fold: 
\begin{itemize}
\item We propose a new concept of surprise novelty for intrinsic motivation.
We argue that it reflects better the environment originality than
surprise norm (see motivating graphics Fig. \ref{fig:Montezuma-Revenge:-surprise}). 
\item We design a novel memory system, named Surprise Memory (SM) that consists
of an autoencoder associative memory and an attention-based episodic
memory. Our two-memory system estimates surprise novelty within and
across episodes.
\end{itemize}
In our experiments, the SM helps RND \cite{burda2018exploration}
perform well in our challenging noisy-TV problem while RND alone performs
poorly. Not only with RND, we consistently demonstrate significant
performance gain when coupling three different SGs with our SM in
sparse-reward tasks. Finally, in hard exploration Atari games, we
boost the scores of 2 strong SGs, resulting in better performance
under the low-sample regime.

\section{Methods}

\begin{figure*}
\begin{centering}
\includegraphics[width=0.9\linewidth]{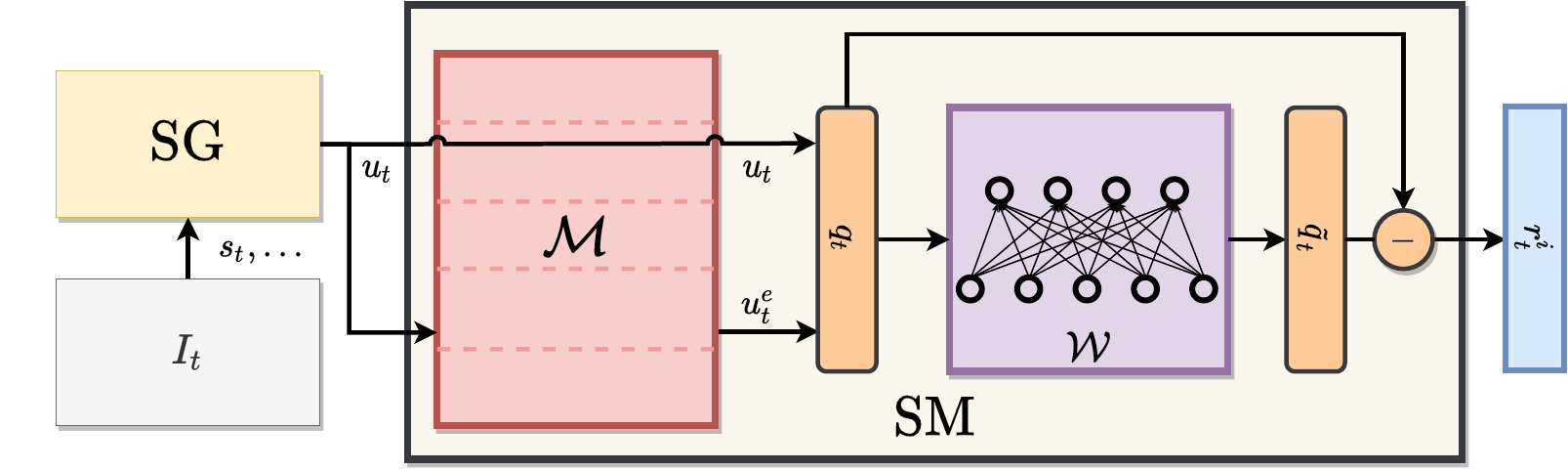}
\par\end{centering}
\caption{Surprise Generator+Surprise Memory (SG+SM). The SG takes input $I_{t}$
from the environment to estimate the surprise $u_{t}$ at state $s_{t}$.
The SM consists of two modules: an episodic memory ($\mathcal{M}$)
and an autoencoder network ($\mathcal{W}$). $\mathcal{M}$ is slot-based,
storing past surprises within the episode. At timestep $t$, given
surprise $u_{t}$, $\mathcal{M}$ retrieves read-out $u_{t}^{e}$
to form a query surprise $q_{t}=\left[u_{t}^{e},u_{t}\right]$ to
$\mathcal{W}$. $\mathcal{W}$  reconstructs the query and
takes the reconstruction error (surprise novelty) as the intrinsic
reward $r_{t}^{i}$. \label{fig:Surprise-Generator-Surprise-Memo}}
\end{figure*}

\subsection{Surprise Novelty}

Surprise is the difference between expectation and observation \cite{ekman1994nature}.
If a surprise repeats, it is no longer a surprise. Based on this intuition,
we hypothesize that surprises can be characterized by their novelties,
and an agent's curiosity is driven by the surprise novelty rather
than the surprising magnitude. Moreover, surprise novelty should be
robust against noises: it is small even for random observations. For
example, watching a random-channel TV can always be full of surprises
as we cannot expect which channel will appear next. However, the agent
should soon find it boring since the surprise of random noises reoccurs
repeatedly, and the channels are entirely unpredictable. 

We propose using a memory-augmented neural network (MANN) to measure
surprise novelty. The memory remembers past surprise patterns, and
if a surprise can be retrieved from the memory, it is not novel, and
the intrinsic motivation should be small. The memory can also be viewed
as a reconstruction network. The network can pass its inputs through
for random, pattern-free surprises, making them retrievable. Surprise
novelty has an interesting property: if some event is unsurprising
(the expectation-reality residual is $\overrightarrow{0}$), its surprise
($\overrightarrow{0}$ with norm $0$) is always perfectly retrievable
(surprise novelty is $0$). In other words, low surprise norm means
low surprise novelty. On the contrary, high surprise norm can have
little surprise novelty as long as the surprise can be retrieved from
the memory either through associative recall or pass-through mechanism.
Another property is that the variance of surprise is generally lower
than that of observation (state), potentially making the learning
on surprise space easier. This property is formally stated as follows.
\begin{prop}
Let $X$ and $U$ be random variables representing the observation
and surprise at the same timestep, respectively. Under an imperfect
SG, the following inequality holds: 

\[
\forall i:\,\,\left(\sigma_{i}^{X}\right)^{2}\geq\left(\sigma_{i}^{U}\right)^{2}
\]
where $\left(\sigma_{i}^{X}\right)^{2}$ and $\left(\sigma_{i}^{U}\right)^{2}$denote
the $i$-th diagonal elements of $\mathrm{{var}}(X)$ and $\mathrm{{var}}(U),$
respectively.
\end{prop}

\begin{proof}
See Appendix \ref{subsec:Variance-Inequality}. 
\end{proof}

\subsection{Surprise Generator}

Since our MANN requires surprises for its operation, it is built upon
a prediction model, which will be referred to as Surprise Generators
(SG). In this paper, we adopt many well-known SGs (e.g. RND \cite{burda2018exploration}
and ICM \cite{pathak2017curiosity}) to predict the observation, compute
the surprise $u_{t}$ for every step in the environment.
The surprise norm is the Euclidean distance between the expectation
and the reality:

\begin{equation}
\left\Vert u_{t}\right\Vert =\left\Vert SG\left(I_{t}\right)-O_{t}\right\Vert \label{eq:ut}
\end{equation}
where $u_{t}\in\mathbb{R}^{n}$ is the surprise vector of size $n$,
$I_{t}$ the input of the SG at step $t$ of the episode, $SG\left(I_{t}\right)$
and $O_{t}$ the SG's prediction and the observation target, respectively.
The input $I_{t}$ is specific to the SG architecture choice, which
can be the current ($s_{t}$) or previous state, action ($s_{t-1},a_{t}$).
The observation target $O_{t}$ is usually a transformation (can be
identical or random) of the current state $s_{t}$, which serves as
the target for the SG's prediction. The SG is usually trained to minimize:
\begin{equation}
\mathcal{L}_{SG}=\mathbb{E}_{t}\left[\left\Vert u_{t}\right\Vert \right]\label{eq:surlss}
\end{equation}
Here, predictable observations have minor prediction errors or little
surprise. One issue is that a great surprise norm can be simply due
to noisy or distractive observations. Next, we propose a remedy for
this problem.

\subsection{Surprise Memory \label{subsec:Surprise-Memory}}

The surprise generated by the SG is stored and processed by a memory
network dubbed Surprise Memory (SM). It consists of an episodic memory
$\mathcal{M}$ and an\emph{ }autoencoder network $\mathcal{W}$, jointly
optimized to reconstruct any surprise. At each timestep, the SM receives
a surprise $u_{t}$ from the SG module and reads content $u_{t}^{e}$
from the memory $\mathcal{M}$.\emph{ }$\left\{ u_{t}^{e},u_{t}\right\} $
forms a surprise query $q_{t}$ to $\mathcal{W}$ to retrieve the
reconstructed $\tilde{q}_{t}$. This reconstruction will be used to
estimate the novelty of surprises forming intrinsic rewards $r_{t}^{i}$.
Fig. \ref{fig:Surprise-Generator-Surprise-Memo} summarizes the operations
of the components of our proposed method. Our 2 memory design effectively
recovers surprise novelty by handling intra and inter-episode surprise
patterns thanks to $\mathcal{M}$ and $\mathcal{W}$, respectively.
$\mathcal{M}$ can quickly adapt and recall surprises that occur within
an episode. $\mathcal{W}$ is slower and focuses more on consistent
surprise patterns across episodes during training.

Here the query $q_{t}$ can be directly set to the surprise $u_{t}$.
However, this ignores the rapid change in surprise within an episode.
Without $\mathcal{M}$, when the SG and $\mathcal{W}$ are fixed (during
interaction with environments), their outputs $u_{t}$ and $\tilde{q}_{t}$
stay the same for the same input $I_{t}$. Hence, the intrinsic reward
$r_{t}^{i}$ also stays the same. It is undesirable since when the
agent observes the same input at different timesteps (e.g., $I_{1}=I_{2}$),
we expect its curiosity should decrease in the second visit ($r_{1}^{i}<$$r_{2}^{i}$).
Therefore, we design SM with $\mathcal{M}$ to fix this issue. 

\textbf{The episodic memory $\mathcal{M}$} stores representations
of surprises that the agent encounters during an episode. For simplicity,
$\mathcal{M}$ is implemented as a first-in-first-out queue whose
size is fixed as $N$. Notably, the content of $\mathcal{M}$ is
wiped out at the end of each episode. Its information is limited to
a single episode. $\mathcal{M}$ can be viewed as a matrix: $\mathcal{M}\in\mathbb{R}^{N\times d},$
where $d$ is the size of the memory slot. We denote $\mathcal{M}\left(j\right)$
as the $j$-th row in the memory, corresponding to the surprise $u_{t-j}$.
To retrieve from $\mathcal{M}$ a read-out $u_{t}^{e}$ that is close
to $u_{t}$, we perform content-based attention \cite{graves2014neural}
to compute the attention weight as $w_{t}\left(j\right)=\frac{\left(u_{t}Q\right)\mathcal{M}\left(j\right)^{\top}}{\left\Vert \left(u_{t}Q\right)\right\Vert \left\Vert \mathcal{M}\left(j\right)\right\Vert }$.
The read-out from $\mathcal{M}$ is then $u_{t}^{e}=w_{t}\mathcal{M}V\in\mathbb{R}^{n}$.
Here, $Q\in\mathbb{R}^{n\times d}$ and $V\in\mathbb{R}^{d\times n}$
are learnable weights mapping between the surprise and the memory
space. To force the read-out close to $u_{t}$, we minimize: 
\begin{equation}
\mathcal{L_{M}}=\mathbb{E}_{t}\left[\left\Vert u_{t}^{e}-u_{t}\right\Vert \right]\label{eq:lm}
\end{equation}

The read-out and the SG's surprise form the query surprise to $\mathcal{W}$:
$q_{t}=\left[u_{t}^{e},u_{t}\right]\in\mathbb{R}^{2n}$. $\mathcal{M}$
stores intra-episode surprises to assist the autoencoder in preventing
the agent from exploring ``fake surprise'' within the episode. Our
episodic memory formulation is unlike prior KNN-based episodic memory
\cite{badia2019never} because our system learns to enforce the memory
retrieval error to be small. This is critical when the representations
stored in $\mathcal{M}$ can not be easily discriminated if we only
rely on unsupervised distance-based metrics. More importantly, under
our formulation, the retrieval error can still be small even when
the stored items in the memory differ from the query. This helps detect
the ``fake surprise'' when the query surprise seems unlike those
in the memory if considered individually, yet can be approximated
as weighted sum of the store surprises. 

In general, since we optimize the parameters to reconstruct $u_{t}$
using past surprises in the episode, if the agent visits a state whose
surprise is predictable from those in $\mathcal{M}$, $\left\Vert u_{t}^{e}-u_{t}\right\Vert $
should be small. Hence, the read-out context $u_{t}^{e}$ contains
no extra information than $u_{t}$ and reconstructing $q_{t}$ from
$\mathcal{W}$ becomes easier as it is equivalent to reconstructing
$u_{t}$. In contrast, visiting diverse states leads to a more novel
read-out $u_{t}^{e}$ and makes it more challenging to reconstruct
$q_{t}$, generally leading to higher intrinsic reward. 

\textbf{The autoencoder network $\mathcal{W}$} can be viewed as an
associative memory of surprises that persist across episodes. At timestep
$t$ in any episode during training, $\mathcal{W}$ is queried with
$q_{t}$ to produce a reconstructed memory $\tilde{q}_{t}$. The surprise
novelty is determined as:

\begin{equation}
r_{t}^{i}=\left\Vert \tilde{q}_{t}-q_{t}\right\Vert 
\end{equation}
which is the norm of the surprise residual $\tilde{q}_{t}-q_{t}$.
It will be normalized and added to the external reward as an intrinsic
reward bonus. The details of computing and using normalized intrinsic
rewards can be found in Appendix \ref{subsec:Intrinsic-Reward-Integration}.

We implement $\mathcal{W}$ as a feed-forward neural network that
learns to reconstruct its own inputs. The query surprise is encoded
to the weights of the network via backpropagation as we minimize the
reconstruction loss below:

\begin{equation}
\mathcal{L_{\mathcal{W}}}=\mathbb{E}_{t}\left[r_{t}^{i}\right]=\mathbb{E}_{t}\left[\left\Vert \mathcal{W}\left(q_{t}\right)-q_{t}\right\Vert \right]\label{eq:recloss}
\end{equation}
Here, $\tilde{q}_{t}=\mathcal{W}\left(q_{t}\right)$. Intuitively,
it is easier to retrieve non-novel surprises experienced many times
in past episodes. Thus, the intrinsic reward is lower for states that
leads to these familiar surprises. On the contrary, rare surprises
are harder to retrieve, which results in high reconstruction errors
and intrinsic rewards. We note that autoencoder (AE) has been shown
to be equivalent to an associative memory that supports memory encoding
and retrieval through attractor dynamics \cite{radhakrishnan2020overparameterized}.
Unlike slot-based memories, AE has a fixed memory capacity, compresses
information and learns data representations. We could store the surprise
in a slot-based memory across episodes, but the size of this memory
would be autonomous, and the data would be stored redundantly. Hence,
the quality of the stored surprise will reduce as more and more observations
come in. On the other hand, AE can efficiently compress surprises
to latent representations and hold them to its neural weights, and
the surprise retrieval is optimized. Besides, AE can learn to switch
between 2 mechanisms: pass-through and pattern retrieval, to optimally
achieve its objective. We cannot do that with slot-based memory. Readers
can refer to Appendix \ref{subsec:-as-Associative} to see the architecture
details and how $\mathcal{W}$ can be interpreted as implementing
associative memory. 

The whole system SG+SM is trained end-to-end by minimizing the following
loss: $\mathcal{L}=\mathcal{L}_{SG}+\mathcal{L_{M}}+\mathcal{L_{W}}$.
Here, we block the gradients from $\mathcal{L_{W}}$ backpropagated
to the parameters of SG to avoid trivial reconstructions of $q_{t}$.
The pseudocode of our algorithm is presented in Algo. \ref{alg:ir} and \ref{alg:ir_full}.
We note that vector notations in the algorithm are row vectors. For
simplicity, the algorithm assumes 1 actor. In practice, our algorithm
works with multiple actors and mini-batch training. See Appendix \ref{subsec:Intrinsic-Reward-Integration} for explaination of $\beta$ and $r^{std}$.

\begin{algorithm}[t]
\begin{algorithmic}[1]
\REQUIRE{$u_t$, and our surprise memory $\mathrm{SM}$ consisting of a slot-based memory $\mathcal{M}$, parameters $Q$, $V$,  and a neural network $\mathcal{W}$}
\STATE{Compute $\mathcal{L}_{SG}=\left\Vert u_{t}\right\Vert$}
\STATE{Query $\mathcal{M}$ with $u_t$, retrieve $u_{t}^{e}=w_{t}\mathcal{M}V$ where $w_t$ is the attention weight}
\STATE{Compute $\mathcal{L_{M}}=\Vert u_{t}^{e}-u_{t}.detach()\Vert$}
\STATE{Query $\mathcal{W}$ with $q_t=[u_{t}^{e},u_t]$, retrieve $\tilde{q}_{t}=\mathcal{W}(q_t)$}
\STATE{Compute intrinsic reward $r^i_t=L_{\mathcal{W}}=\left\Vert \tilde{q}_{t}-q_{t}.detach()\right\Vert$}
\RETURN{$\mathcal{L}_{SG}$, $\mathcal{L_{M}}$, $L_{\mathcal{W}}$ }
\end{algorithmic}

\caption{Intrinsic rewards computing via SG+SM framework.\label{alg:ir}}
\end{algorithm}
\begin{algorithm}[t]
\begin{algorithmic}[1]
\REQUIRE{buffer, policy $\pi_{\theta}$, surprise-based predictor $\mathrm{SG}$, and our surprise memory $\mathrm{SM}$ consisting of a slot-based memory $\mathcal{M}$, parameters $Q$, $V$, and a neural network $\mathcal{W}$}
\STATE{Initialize $\pi_{\theta}$, $\mathrm{SG}$, $Q$, $\mathcal{W}$}
\FOR{$iteration=1,2,...$}
\FOR{$t=1,2,...T$} 
\STATE{Execute policy $\pi_{\theta}$ to collect $s_t$, $a_t$, $r_t$, forming input $I_t={s_t, ...}$ and target $O_t$}
\STATE{Compute surprise $u_t=SG\left(I_{t}\right)-O_t.detach()$ (Eq. $\ref{eq:ut}$)}
\STATE{Compute intrinsic reward $r^i_t$ using Algo. $\ref{alg:ir}$}
\STATE{Compute final reward $r_t \leftarrow r_t+\beta r^i_t/r^{std}_t$}
\STATE{Add $(I_t,O_t,s_{t-1},s_t,a_t,r_t)$ to buffer}
\STATE{Add $u_tQ$ to $\mathcal{M}$}
\STATE{$\mathbf{if}$ done episode $\mathbf{then}$ clear $\mathcal{M}$}
\ENDFOR
\FOR{$k=1,2,..,K$} 
\STATE{Sample $I_t$, $O_t$ from buffer}
\STATE{Compute surprise $u_t=SG\left(I_{t}\right)-O_t.detach()$ (Eq. $\ref{eq:ut}$)}
\STATE{Compute $\mathcal{L}_{SG}$, $\mathcal{L_{M}}$, $L_{\mathcal{W}}$ using Algo. $\ref{alg:ir}$}
\STATE{Update $\mathrm{SG}$, $Q$ and $\mathcal{W}$ by minimizing the loss $\mathcal{L}=\mathcal{L}_{SG}+\mathcal{L_{M}}+\mathcal{L_{W}}$}
\STATE{Update $\pi_{\theta}$ with sample $(s_{t-1},s_t,a_t,r_t)$ from buffer using backbone algorithms}
\ENDFOR
\ENDFOR
\end{algorithmic}

\caption{Jointly training SG+SM and the policy. \label{alg:ir_full}}
\end{algorithm}

\section{Experimental Results}

\begin{figure*}
\begin{centering}
\includegraphics[width=1\linewidth]{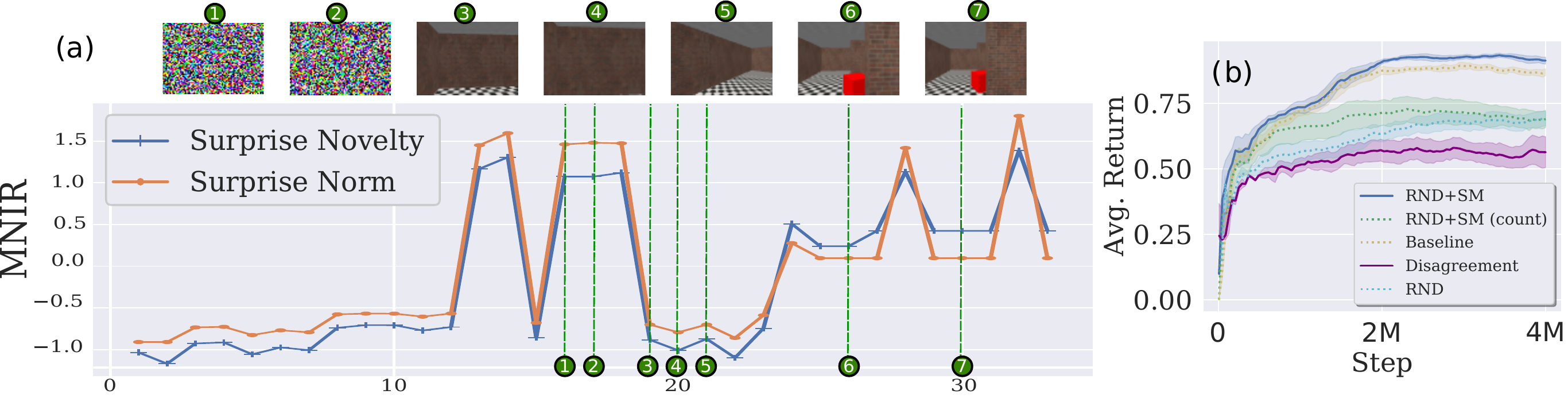}
\par\end{centering}
\caption{Noisy-TV: (a) mean-normalized intrinsic reward (MNIR) produced by
RND and RND+SM at 7 selected steps in an episode. (b) Average task
return (mean$\pm$std. over 5 runs) over 4 million training steps.\label{fig:ntv_view}}
\end{figure*}

\subsection{Noisy-TV: Robustness against Noisy Observations}

We use Noisy-TV, an environment designed to fool exploration methods
\cite{burda2018exploration,savinov2018episodic}, to confirm that
our method can generate intrinsic rewards that (1) are more robust
to noises and (2) can discriminate rare and common observations through
surprise novelty. We simulate this problem by employing a 3D maze
environment with a random map structure. The TV is not fixed in specific
locations in the maze to make it more challenging. Instead, the agent
\textquotedblleft brings\textquotedblright{} the TV with it and can
choose to watch TV anytime. Hence, there are three basic actions (turn
left, right, and move forward) plus an action: watch TV. When taking
this action, the agent will see a white noise image sampled from standard
normal distribution and thus, the number of TV channels can be considered
infinity. The agent's state is an image of its viewport, and its goal
is to search for a red box randomly placed in the maze (+1 reward
if the agent reaches the goal). 

The main baseline is RND \cite{burda2018exploration},
a simple yet strong SG that is claimed to obviate the stochastic problems
of Noisy-TV. Our SG+SM model uses RND as the SG, so we name it RND+SM. 
The source code can be accessed on our GitHub repository at \url{https://github.com/thaihungle/SM}.
Since our model and the baseline share the same RND architecture,
the difference in performance must be attributed to our SM.  We also include
a disagreement-based method  (Disagreement) \cite{pathak2019self}, which leverages the variance of 
multiple dynamic predictors as the intrinsic reward. Finally, as a reference, we use PPO without intrinsic reward as the vanilla Baseline, which is not affected by the noisy TV traps. 

The result demonstrates that the proposed SM 
 outperforms other intrinsic motivators by
a significant margin, as shown in Fig. \ref{fig:ntv_view} (b). Notably, SM improves RND performance by around 25 $\%$,
showcasing the impact of surprise novelty.  
The result also shows that RND+SM outperforms the
vanilla Baseline. Although the improvement is moderate (0.9 vs 0.85),
the result is remarkable since the Noisy-TV is designed to fool intrinsic
motivation methods and among all, only RND+SM can outperform the vanilla
Baseline.

Fig. \ref{fig:ntv_view} (a) illustrates the mean-normalized intrinsic
rewards (MNIR)\footnote{See Appendix \ref{subsec:Intrinsic-Reward-Integration} for more information
on this metric.} measured at different states in our Noisy-TV environment. The first
two states are noises, the following three states are common walls,
and the last two are ones where the agent sees the box. The MNIR bars
show that both models are attracted mainly by the noisy TV, resulting
in the highest MNIRs. However, our model with SM suffers less from
noisy TV distractions since its MNIR is lower than RND's. We speculate
that SM is able to partially reconstruct the white-noise surprise
via the pass-through mechanism, making the normalized surprise novelty
generally smaller than the normalized surprise norm in this case.
That mechanism is enhanced in SM with surprise reconstruction (see
Appendix \ref{subsec:Noisy-TV} for explanation).

On the other hand, when observing red box, RND+SM shows higher MNIR
than RND. The difference between MNIR for common and rare states is
also more prominent in RND+SM than in RND because RND prediction is
not perfect even for common observations, creating relatively significant
surprise norms for seeing walls. The SM fixes that issue by remembering
surprise patterns and successfully retrieving them, producing much
smaller surprise novelty compared to those of rare events like seeing
red box. Consequently, the agent with SM outperforms the other by
a massive margin in task rewards (Fig. \ref{fig:ntv_view} (b)). 

As we visualize the number of watching TV actions and the value of
the intrinsic reward by RND+SM and RND over training time, we realize
that RND+SM helps the agent take fewer watching actions and thus,
collect smaller amounts of intrinsic rewards compared to RND. See more in Appendix \ref{subsec:Noisy-TV}.

\begin{table*}
\begin{centering}
\begin{tabular}{cccccccccc}
\hline 
\multirow{2}{*}{{\normalsize{}Task}} & \multicolumn{1}{c}{{\normalsize{}w/o intrinsic}} & \multicolumn{2}{c}{{\normalsize{}RND}} & \multicolumn{2}{c}{{\normalsize{}ICM}} & \multicolumn{2}{c}{{\normalsize{}NGU}} & \multicolumn{2}{c}{{\normalsize{}AE}}\tabularnewline
 & {\normalsize{}reward} & {\normalsize{}w/o SM } & {\normalsize{}w/ SM} & {\normalsize{}w/o SM } & {\normalsize{}w/ SM} & {\normalsize{}w/o SM } & {\normalsize{}w/ SM} & {\normalsize{}w/o SM } & {\normalsize{}w/ SM}\tabularnewline
\hline 
{\normalsize{}KD} & {\normalsize{}0.0$\pm$0.0} & {\normalsize{}48.3$\pm$26} & \emph{\normalsize{}79.3$\pm$4} & {\normalsize{}5.9$\pm$5} & {\normalsize{}4.7$\pm$3} & {\normalsize{}64.4}\emph{\normalsize{}$\pm$}{\normalsize{}3} & \emph{\normalsize{}83.4$\pm$4} & {\normalsize{}1.4$\pm$1} & \textbf{\emph{\normalsize{}91.2$\pm$6}}\tabularnewline
{\normalsize{}DO} & {\normalsize{}-27.0$\pm$0.7} & {\normalsize{}-13.6$\pm$8} & \textbf{\emph{\normalsize{}70.8$\pm$11}} & {\normalsize{}-27.7$\pm$2} & \emph{\normalsize{}43.6$\pm$16} & {\normalsize{}-23.9$\pm$3} & \emph{\normalsize{}48.6$\pm$28} & {\normalsize{}-5.1$\pm$2} & \emph{\normalsize{}67.5$\pm$13}\tabularnewline
{\normalsize{}LC} & {\normalsize{}78.0$\pm$1.7} & {\normalsize{}25.0$\pm$35} & \emph{\normalsize{}71.1$\pm$5} & {\normalsize{}56.2$\pm$40} & \textbf{\emph{\normalsize{}84.6$\pm$1}} & {\normalsize{}42.2$\pm$40} & \emph{\normalsize{}69.5$\pm$5} & {\normalsize{}29.0$\pm$6/} & \emph{\normalsize{}70.9$\pm$2}\tabularnewline
\hline 
\end{tabular}
\par\end{centering}
~

\caption{MiniGrid: test performance after 10 million training steps. The numbers
are average task return$\times100$ over 128 episodes (mean$\pm$std.
over 5 runs). Bold denotes the best results on each task. Italic denotes
that SG+SM is better than SG regarding Cohen effect size less than
0.5. \label{tab:Minigrid:-test-performance}}
\end{table*}

\begin{figure*}
\begin{centering}
\includegraphics[width=0.8\linewidth]{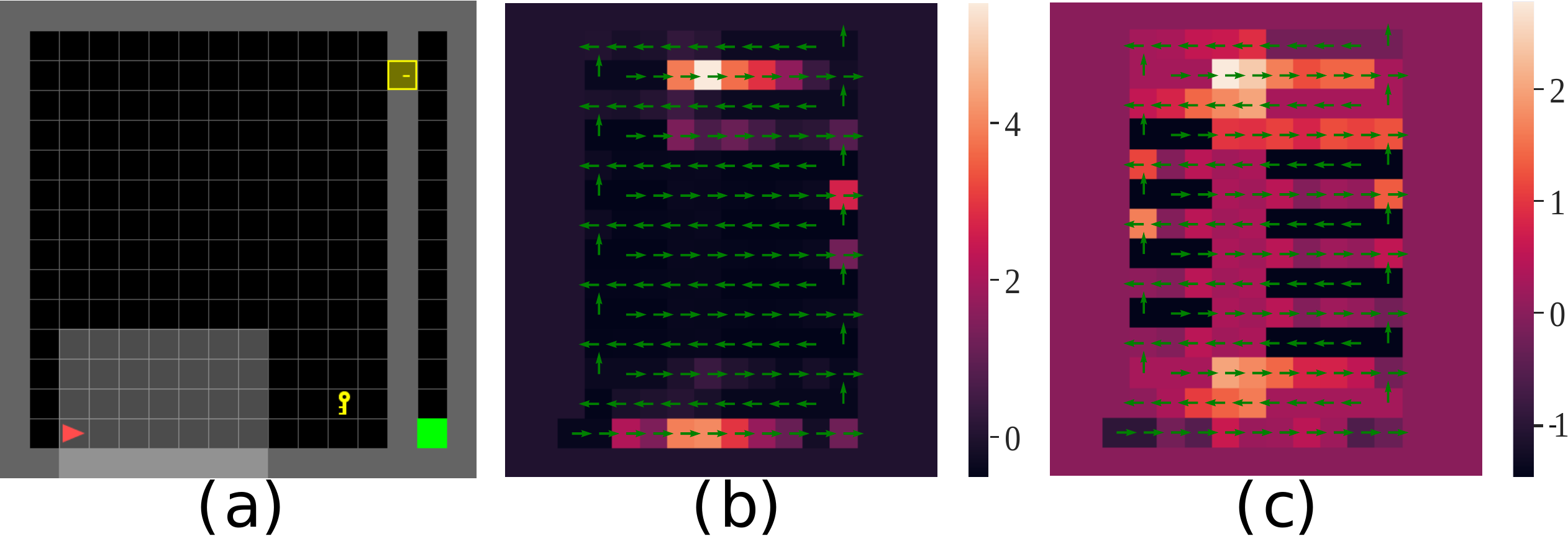}
\par\end{centering}
\caption{Key-Door: (a) Example map in Key-Door where the light window is the
agent's view window (state). MNIR produced for each cell in a manually
created trajectory for RND+SM (b) and RND (c). The green arrows denote
the agent's direction at each location. The brighter the cell, the
higher MNIR assigned to the corresponding state. \label{fig:kd_map}}
\end{figure*}

\subsection{MiniGrid: Compatibility with Different Surprise Generators\label{subsec:MiniGrid:-Compatibility-with}}

We show the versatility of our framework SG+SM by applying SM to 4
SG backbones: RND \cite{burda2018exploration}, ICM \cite{pathak2017curiosity},
NGU \cite{badia2019never} and autoencoder-AE (see Appendix \ref{subsec:MiniGrida}
for implementation details). We test the models on three tasks from
MiniGrid environments: Key-Door (KD), Dynamic-Obstacles (DO) and Lava-Crossing
(LC) \cite{gym_minigrid}. If the agent reaches the goal in the tasks,
it receives a +1 reward. Otherwise, it can be punished with negative
rewards if it collides with obstacles or takes too much time to finish
the task. These environments are not as stochastic as the Noisy-TV but
they still contain other types of distraction. For example, in KD,
the agent can be attracted to irrelevant actions such as going around
to drop and pick the key. In DO, instead of going to the destination,
the agent may chase obstacle balls flying around the map. In LC the
agent can commit unsafe actions like going near lava areas, which
are different from typical paths. In any case, due to reward sparsity,
intrinsic motivation is beneficial. However, surprise alone may not
be enough to guide an efficient exploration since the observation
can be too complicated for SG to minimize its prediction error. Thus,
the agent quickly feels surprised, even in unimportant states.

Table \ref{tab:Minigrid:-test-performance} shows the average returns
of the models for three tasks. The Baseline is the PPO backbone trained
without intrinsic reward. RND, ICM, NGU and AE are SGs providing the
PPO with surprise-norm rewards while our method SG+SM uses surprise-novelty
rewards. The results demonstrate that models with SM often outperform
SG significantly and always contain the best performers. Notably,
in the LC task, SGs hinder the performance of the Baseline because
the agents are attracted to dangerous vivid states, which are hard
to predict but cause the agent's death. The SM models avoid this issue
and outperform the Baseline for the case of ICM+SM. Compared to AE,
which computes intrinsic reward based on the novelty of the state,
AE+SM shows a much higher average score in all tasks. That manifests
the importance of modeling the novelty of surprise instead of states. 

Regarding the compatibility of our SM with different SGs, we realize that if the SG is strong, SG+SM tends to have better performance. For example, in LC task, ICM is the best SG, resulting ICM+SM being the best performer. The exception is the AE SG in KD task. We speculate that AE generates intrinsic rewards as the reconstruction error, which can be more sensitive to the state representation of specific tasks. In KD, the viewport of the agent is almost empty. The state representations look similar most of the time, leading to similar AE's reconstruction errors (surprise norm), even for special events (pick the key). Therefore, AE performance is just slightly above the no-IR baseline. The good thing is that other noisy states like throwing the key also receive low IR. When equipped with SM, AE+SM differentiates the reconstruction residual vectors, which can vary even when they have similar norms (Appendix Fig. 8). Therefore, AE+SM can assign high IR to special events while still providing low IR to noisy states, which is optimal in this case.

To analyze the difference between the SG+SM and SG's MNIR structure,
we visualize the MNIR for each cell in the map of Key-Door in
Figs. \ref{fig:kd_map} (b) and (c). We create a synthetic trajectory
that scans through all the cells in the big room on the left and,
at each cell, uses RND+SM and RND models to compute the corresponding
surprise-norm and surprise-novelty MNIRs, respectively. As shown in
Fig. \ref{fig:kd_map} (b), RND+SM selectively identifies truly surprising
events, where only a few cells have high surprise-novelty MNIR. Here,
we can visually detect three important events that receive the most
MNIR: seeing the key (bottom row), seeing the door side (in the middle
of the rightmost column) and approaching the front of the door (the
second and fourth rows). Other less important cells are assigned very
low MNIR. On the contrary, RND often gives high surprise-norm MNIR
to cells around important ones, which creates a noisy MNIR map as
in Fig. \ref{fig:kd_map} (c). As a result, RND's performance is better
than Baseline, yet far from that of RND+SM. Another analysis of how
surprise novelty discriminates against surprises with similar norms
is given in Appendix Fig. \ref{tnse-1}.

\begin{table*}
\begin{centering}
{\normalsize{}}%
\begin{tabular}{ccccc|cccc}
\hline 
\multirow{1}{*}{{\normalsize{}Task}} & {\normalsize{}EMI$^{\spadesuit}$} & \multirow{1}{*}{{\normalsize{}LWM$^{\spadesuit}$}} & {\normalsize{}RND$^{\spadesuit}$} & {\normalsize{}NGU$^{\diamondsuit}$} & {\normalsize{}LWM$^{\diamondsuit}$} & {\normalsize{}LWM+SM$^{\diamondsuit}$} & {\normalsize{}RND$^{\diamondsuit}$} & {\normalsize{}RND+SM$^{\diamondsuit}$}\tabularnewline
\hline 
{\normalsize{}Freeway} & \textbf{\normalsize{}33.8} & {\normalsize{}30.8} & {\normalsize{}33.3} & {\normalsize{}2.6} & {\normalsize{}31.1} & {\normalsize{}31.6} & {\normalsize{}22.2} & {\normalsize{}22.2}\tabularnewline
{\normalsize{}Frostbite} & {\normalsize{}7002} & {\normalsize{}8409} & {\normalsize{}2227} & {\normalsize{}1751} & {\normalsize{}8598} & \textbf{\emph{\normalsize{}10258}} & {\normalsize{}2628} & \emph{\normalsize{}5073}\tabularnewline
{\normalsize{}Venture} & {\normalsize{}646} & {\normalsize{}998} & {\normalsize{}707} & {\normalsize{}790} & {\normalsize{}985} & \textbf{\emph{\normalsize{}1381}} & {\normalsize{}1081} & \emph{\normalsize{}1119}\tabularnewline
{\normalsize{}Gravitar} & {\normalsize{}558} & {\normalsize{}1376} & {\normalsize{}546} & {\normalsize{}839} & {\normalsize{}1,242} & \textbf{\emph{\normalsize{}1693}} & {\normalsize{}739} & \emph{\normalsize{}987}\tabularnewline
{\normalsize{}Solaris} & {\textbf{\normalsize{}2688}} & {\normalsize{}1268} & {\normalsize{}2051} & {\normalsize{}1103} & {\normalsize{}1839} & \emph{\normalsize{}2065} & {\normalsize{}2206} & \emph{\normalsize{}2420}\tabularnewline
{\normalsize{}Montezuma} & {\normalsize{}387} & {\normalsize{}2276} & {\normalsize{}377} & {\normalsize{}2062} & {\normalsize{}2192} & {\normalsize{}2269} & {\normalsize{}2475} & \textbf{\emph{\normalsize{}5187}}\tabularnewline
\hline 
{\normalsize{}Norm. Mean} & {\normalsize{}61.4} & {\normalsize{}80.6} & {\normalsize{}42.2} & {\normalsize{}31.2} & {\normalsize{}80.5} & \textbf{\emph{\normalsize{}97.0}} & {\normalsize{}50.7} & \emph{\normalsize{}74.8}\tabularnewline
{\normalsize{}Norm. Median} & {\normalsize{}34.9} & {\normalsize{}60.8} & {\normalsize{}32.7} & {\normalsize{}33.1} & {\normalsize{}66.5} & \emph{\normalsize{}83.7} & {\normalsize{}58.3} & \textbf{\emph{\normalsize{}84.6}}\tabularnewline
\hline 
\end{tabular}{\normalsize\par}
\par\end{centering}
~

\caption{Atari: average return over 128 episodes after 50 million training
frames (mean over 5 runs). $\spadesuit$ is from a prior work \cite{ermolov2020latent}.
$\diamondsuit$ is our run. The last two rows are mean and median
human normalized scores. Bold denotes best results. Italic denotes
that SG+SM is significantly better than SG regarding Cohen effect
size less than 0.5. \label{tab:atari6}}
\end{table*}

\subsection{Atari: Sample-efficient Benchmark }

We adopt the sample-efficiency Atari benchmark \cite{kim2019emi}
on six hard exploration games where the training budget is only 50
million frames. We use our SM to augment 2 SGs: RND \cite{burda2018exploration}
and LWM \cite{ermolov2020latent}. Unlike RND, LWM uses a recurrent
world model and forward dynamics to generate surprises. Details of
the SGs, training and evaluation are in Appendix \ref{subsec:aAtari}.
We run the SG and SG+SM in the same codebase and setting. Table \ref{tab:atari6}
reports our and representative results from prior works, showing SM-augmented
models outperform their SG counterparts in all games (same codebase).
In Frostbite and Montezuma Revenge, RND+SM's score is almost twice
as many as that of RND. For LWM+SM, games such as Gravitar and Venture
observe more than 40\% improvement. Overall, LWM+SM and RND+SM achieve
the best mean and median human normalized score, improving 16\% and
22\% w.r.t the best SGs, respectively. Notably, RND+SM shows significant
improvement for the notorious Montezuma Revenge. We also test the
episodic memory baseline NGU to verify whether episodic memory on
the state level is good enough in such a challenging benchmark. The result
shows that NGU does not perform well within 50 million training frames,
resulting in human normalized scores much lower than our LWM+SM and
RND+SM. 

We also verify the benefit of the SM in the long run for Montezuma
Revenge and Frostbite. As shown in Fig. \ref{fig:atari200} (a,b),
RND+SM still significantly outperforms RND after 200 million training
frames, achieving average scores of 10,000 and 9,000, respectively.
The result demonstrates the scalability of our proposed method. When
using RND and RND+SM to compute the average MNIR in several rooms
in Montezuma Revenge (Fig. \ref{fig:Montezuma-Revenge:-surprise}),
we realize that SM makes MNIR higher for surprising events in rooms
with complex structures while depressing the MNIR of fake surprises
in dark rooms. Here, even in the dark room, the movement of agents
(human or spider) is hard to predict, leading to a high average MNIR.
On the contrary, the average MNIR of surprise novelty is reduced if
the prediction error can be recalled from the memory. 

Finally, measuring the running time of the models, we notice little
computing overhead caused by our SM. On our Nvidia A100 GPUs, LWM
and LWM+SM's average time for one 50M training are 11h 38m and 12h
10m, respectively. For one 200M training, RND and RND+SM's average
times are 26h 24m and 28h 1m, respectively. These correspond to only
7\% more training time while the performance gap is significant (4000
scores).

\begin{figure*}
\begin{centering}
\includegraphics[width=1\linewidth]{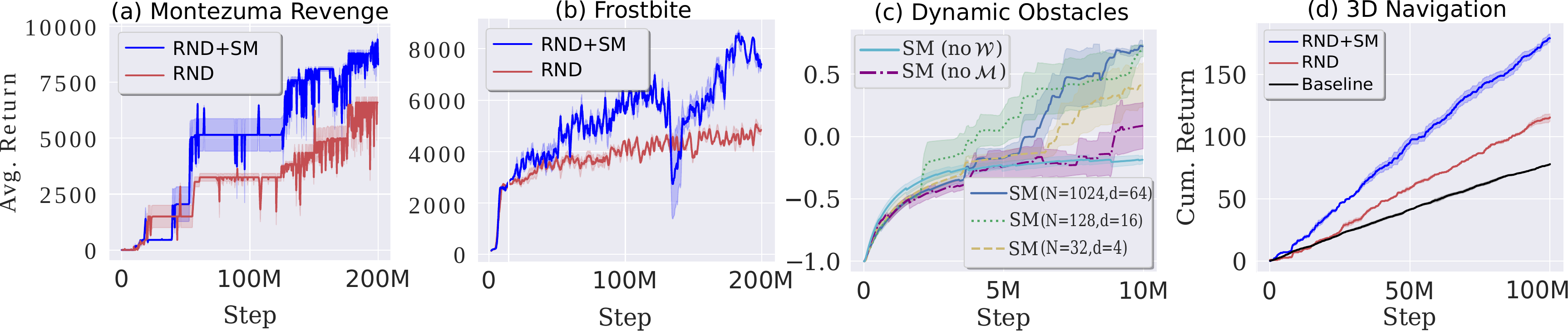}
\par\end{centering}
\caption{(a,b) Atari long runs over 200 million frames: average return over
128 episodes. (c) Ablation study on SM's components. (d) MiniWorld
exploration without task reward: Cumulative task returns over 100
million training steps for the hard setting. The learning curves are
mean$\pm$std. over 5 runs. \label{fig:atari200}}
\end{figure*}

\subsection{Ablation Study\label{subsec:Ablation-Study}}

\textbf{Role of Memories} Here, we use Minigrid's Dynamic-Obstacle
task to study the role of $\mathcal{M}$ and $\text{\ensuremath{\mathcal{W}}}$
in the SM (built upon RND as the SG). Disabling $\mathcal{W}$, we
directly use $\left\Vert q_{t}\right\Vert =\left\Vert \left[u_{t}^{e},u_{t}\right]\right\Vert $
as the intrinsic reward, and name this version: SM (no $\mathcal{W}$).
To ablate the effect of $\mathcal{M}$, we remove $u_{t}^{e}$ from
$q_{t}$ and only use $q_{t}=u_{t}$ as the query to $\mathcal{W}$,
forming the version: SM (no $\mathcal{M}$). We also consider different
episodic memory capacities and slot sizes $N$-$d$$=\left\{ 32-4,128-16,1024-64\right\} $.
As $N$ and $d$ increase, the short-term context expands and more
past surprise information is considered in the attention. In theory,
a big $\mathcal{M}$ is helpful to capture long-term and more accurate
context for constructing the surprise query. 

Fig. \ref{fig:atari200} (c) depicts the performance curves of the
methods after 10 million training steps. SM (no $\mathcal{W}$) and
SM (no $\mathcal{M}$) show weak signs of learning, confirming the
necessity of both modules in this task. Increasing $N$-$d$ from
$32-4$ to $1024-64$ improves the final performance. However, $1024-64$
is not significantly better than $128-16$, perhaps because it is
unlikely to have similar surprises that are more than 128 steps apart.
Thus, a larger attention span does not provide a benefit. As a result,
we keep using $N=128$ and $d=16$ in all other experiments for faster
computing. We also verify the necessity of $\mathcal{M}$ and $\mathcal{W}$
in Montezuma Revenge, Frostbite, Venture and illustrate how $\mathcal{M}$ generates lower
MNIR when 2 similar event occurs in the same episode in Key-Door (see
Appendix \ref{subsec:aAblation-study}). 

\textbf{No Task Reward} In this experiment, we remove task rewards
and merely evaluate the agent's ability to explore using intrinsic
rewards. The task is to navigate 3D rooms and get a +1 reward for
picking an object \cite{gym_miniworld}. The state is the agent's
image view, and there is no noise. Without task rewards, it is crucial
to maintain the agent's interest in the unique events of seeing the objects.
In this partially observable environment, surprise-prediction methods
may struggle to explore even without noise due to a lack of information
for good predictions, leading to usually high prediction errors. For
this testbed, we evaluate random exploration agent (Baseline), RND
and RND+SM in 2 settings: 1 room with three objects (easy), and 4
rooms with one object (hard).

To see the difference among the models, we compare the cumulative
task rewards over 100 million steps (see Appendix \ref{subsec:MiniWorlda}
for details). RND is even worse than Baseline in the easy setting
because predicting causes high biases (intrinsic rewards) towards
the unpredictable, hindering exploration if the map is simple. In
contrast, RND+SM uses surprise novelty, generally showing smaller
intrinsic rewards (see Appendix Fig. \ref{fig:norall} (right)). Consequently,
our method consistently demonstrates significant improvements over
other baselines (see Fig. \ref{fig:atari200} (d) for the hard setting).

\section{Related works}

Intrinsic motivation approaches usually give the agent reward bonuses
for visiting novel states to encourage exploration. The bonus is proportional
to the mismatch between the predicted and reality, also known as surprise
\cite{schmidhuber2010formal}. One kind of predictive model is the
dynamics model, wherein the surprise is the error of the models as
predicting the next state given the current state and action \cite{achiam2017surprise,stadie2015incentivizing}.
One critical problem of these approaches is the unwanted bias towards
transitions where the prediction target is a stochastic function of
the inputs, commonly found in partially observable environments. Recent
works focus on improving the features of the predictor's input by
adopting representation learning mechanisms such as inverse dynamics
\cite{pathak2017curiosity}, variational autoencoder, random/pixel
features \cite{burda2018large}, or whitening transform \cite{ermolov2020latent}.
Although better representations may improve the reward bonus, they
cannot completely solve the problem of stochastic dynamics and thus,
fail in extreme cases such as the noisy-TV problem \cite{burda2018exploration}. 

Besides dynamics prediction, several works propose to predict other
quantities as functions of the current state by using autoencoder
\cite{nylend2017data}, episodic memory \cite{savinov2018episodic},
and random network \cite{burda2018exploration}. Burda et al. (2018)
claimed that using a deterministic random target network is beneficial
in overcoming stochasticity issues. Other methods combine this idea
with episodic memory and other techniques, achieving good results
in large-scale experiments \cite{badia2020agent57,badia2019never}.
From an information theory perspective, the notation of surprise can
be linked to information gain or uncertainty, and predictive models
can be treated as parameterized distributions \cite{achiam2017surprise,houthooft2016vime,still2012information}.
Furthermore, to prevent the agent from unpredictable observations,
the reward bonus can be measured by the progress of the model's prediction
\cite{achiam2017surprise,lopes2012exploration,schmidhuber1991curious}
or disagreement through multiple dynamic models \cite{pathak2019self,sekar2020planning}.
However, these methods are complicated and hard to scale, requiring
heavy computing. A different angle to handle stochastic observations
during exploration is surprsie minimization \cite{berseth2020smirl,rhinehart2021intrinsic}.
In this direction, the agents get bigger rewards for seeing more familiar
states. Such a strategy is somewhat opposite to our approach and suitable
for unstable environments where the randomness occurs separately from
the agents' actions.

These earlier works use surprise as an
incentive for exploration and differ from our principle that utilizes
surprise novelty. Also, our work augments these existing works with
a surprise memory module and can be used as a generic plug-in improvement
for surprise-based models. We note that our memory formulation differs
from the memory-based novelty concept using episodic memory \cite{badia2019never},
momentum memory \cite{fang2022image}, or counting \cite{bellemare2016unifying,tang2017exploration}
because our memory operates on the surprise level, not the state level.
Moreover, our utilization of memory to store surprise novelty represents
 a novel approach compared to other memory-based RL
 methods that typically use memory for storing trajectories \cite{le2021model}, policies \cite{le2022learning,le2022neurocoder}, and hyperparameters \cite{le2022episodic}. 
In our work, exploration is discouraged not only in frequently visited
states but also in states whose surprises can be reconstructed using
SM. Our work provides a more general and learnable novelty detection
mechanism, which is more flexible than counting lookup table.

\section{Discussion}

This paper presents Surprise Generator-Surprise Memory (SG+SM) framework
to compute surprise novelty as an intrinsic motivation for the reinforcement
learning agent. Exploring with surprise novelty is beneficial when
there are repeated patterns of surprises or random observations. In Noisy-TV, our SG+SM can harness the agent's
tendency to visit noisy states such as watching random TV channels
while encouraging it to explore rare events with distinctive surprises.
We empirically show that our SM can supplement four surprise-based
SGs to achieve more rewards in fewer training steps in three grid-world
environments. In 3D navigation without external reward, our method
significantly outperforms the baselines. On two strong SGs, our SM
also achieve superior results in hard-exploration Atari games within
50 million training frames. Even in the long run, our method maintains
a clear performance gap from the baselines, as shown in Montezuma
Revenge and Frostbite.

\begin{acks}
This research was partially funded by the Australian Government through the Australian Research Council (ARC). 
Prof Venkatesh is the recipient of an ARC Australian Laureate Fellowship (FL170100006).
\end{acks}
\newpage{}
\bibliographystyle{ACM-Reference-Format} 
\balance
\bibliography{ac}

\newpage{}

\section*{Appendix}

\renewcommand\thesubsection{\Alph{subsection}}

\subsection{$\mathcal{W}$ as Associative Memory\label{subsec:-as-Associative}}

This section will connect the associative memory concept to neural
networks trained with the reconstruction loss as in Eq. \ref{eq:recloss}.
We will show how the neural network ($\mathcal{W}$) stores and retrieves
its data. We will use 1-layer feed-forward neural network $W$ to
simplify the analysis, but the idea can extend to multi-layer feed-forward
neural networks. For simplicity, assuming $W$ is a square matrix,
the objective is to minimize the difference between the input and
the output of $W$:

For simplicity, assuming $W$ is a square matrix, the objective is
to minimize the difference between the input and the output of $W$:

\begin{equation}
\mathcal{L}=\left\Vert Wx-x\right\Vert _{2}^{2}\label{eq:apprec}
\end{equation}
Using gradient descent, we update $W$ as follow,

\begin{align*}
W & \leftarrow W-\alpha\frac{\partial\mathcal{L}}{\partial W}\\
 & \leftarrow W-2\alpha\left(Wx-x\right)x^{\top}\\
 & \leftarrow W-2\alpha Wxx^{\top}+2\alpha xx^{\top}\\
 & \leftarrow W\left(I-2\alpha xx^{\top}\right)+2\alpha xx^{\top}
\end{align*}
where $I$ is the identity matrix, $x$ is the column vector. If a
batch of inputs $\left\{ x_{i}\right\} _{i=1}^{B}$ is used in computing
the loss in Eq. \ref{eq:apprec}, at step $t$, we update $W$ as
follows,

\[
W_{t}=W_{t-1}\left(I-\alpha X_{t}\right)+\alpha X_{t}
\]
where $X_{t}=2\sum_{i=1}^{B}x_{i}x_{i}^{\top}$. From $t=0$, after
$T$ updates, the weight becomes

\begin{align}
W_{T} & =W_{0}\prod_{t=1}^{T}\left(I-\alpha X_{t}\right)-\alpha^{2}\sum_{t=2}^{T}X_{t}X_{t-1}\prod_{k=t+1}^{T}\left(I-\alpha X_{k}\right)\nonumber \\
 & +\alpha\sum_{t=1}^{T}X_{t}\label{eq:wt}
\end{align}

Given the form of $X_{t}$, $X_{t}$ is symmetric positive-definite.
Also, as $\alpha$ is often very small (0\textless$\alpha\ll1$ ),
we can show that $\left\Vert I-\alpha X_{t}\right\Vert <1-\lambda_{min}\left(\alpha X_{t}\right)<1$.
This means as $T\rightarrow\infty$, $\left\Vert W_{0}\prod_{t=1}^{T}\left(I-\alpha X_{t}\right)\right\Vert \rightarrow0$
and thus, $W_{T}\rightarrow\alpha^{2}\sum_{t=2}^{T}X_{t}X_{t-1}\prod_{k=t+1}^{T}\left(I-\alpha X_{k}\right)+\alpha\sum_{t=1}^{T}X_{t}$
independent from the initialization $W_{0}$. Eq. \ref{eq:wt} shows
how the data ($X_{t}$) is integrated into the neural network weight
$W_{t}$. The other components such as $\alpha^{2}\sum_{t=2}^{T}X_{t}X_{t-1}\prod_{k=t+1}^{T}\left(I-\alpha X_{k}\right)$
can be viewed as additional encoding noise. Without these components
(by assuming $\alpha$ is small enough), \\
\begin{align*}
W_{T} & \approx\alpha\sum_{t=1}^{T}X_{t}\\
 & =\alpha\sum_{t=1}^{T}\sum_{i=1}^{B}x_{i,t}x_{i,t}^{\top}
\end{align*}
or equivalently, we have the Hebbian update rule:

\[
W\leftarrow W+x_{i,t}\otimes x_{i,t}
\]
where $W$ can be seen as the memory, $\otimes$ is the outer product
and $x_{i,t}$ is the data or item stored in the memory. This memory
update is the same as that of classical associative memory models
such as Hopfield network  and Correlation Matrix Memory (CMM) .

Given a query $q$, we retrieve the value in $W$ as output of the
neural network:

\begin{align*}
q' & =q^{\top}W\\
 & =q^{\top}R+\alpha\sum_{t=1}^{T}qX_{t}\\
 & =q^{\top}R+2\alpha\sum_{t=1}^{T}\sum_{i=1}^{B}q^{\top}x_{i,t}x_{i,t}^{\top}
\end{align*}
where $R=W_{0}\prod_{t=1}^{T}\left(I-\alpha X_{t}\right)-\alpha^{2}\sum_{t=2}^{T}X_{t}X_{t-1}\prod_{k=t+1}^{T}\left(I-\alpha X_{k}\right)$.
If $q$ is presented to the memory $W$ in the past as some $x_{j}$,
$q'$ can be represented as:

\begin{align*}
q' & =q^{\top}R+2\alpha\sum_{t=1}^{T}\sum_{i=1,i\neq j}^{B}q^{\top}x_{i,t}x_{i,t}^{\top}+2\alpha q^{\top}\left(qq^{\top}\right)\\
 & =\underbrace{q^{\top}R}+\underbrace{2\alpha\sum_{t=1}^{T}\sum_{i=1,i\neq j}^{B}q^{\top}x_{i,t}x_{i,t}^{\top}}+2\alpha\left\Vert q\right\Vert q^{\top}\\
 & \mathrm{\mathrm{\,\,\,\,\,\,\,\,\,noise\,\,\,\,\,\,\,\,\,\,\,\,\,\,\,cross}\,\,talk}
\end{align*}

Assuming that the noise is insignificant thanks to small $\alpha$,
we can retrieve exactly $q$ given that all items in the memory are
orthogonal\footnote{By certain transformation, this condition can be reduced to linear
independence}. As a result, after scaling $q'$ with $1/2\alpha$, the retrieval
error ($\left\Vert \frac{q'}{2\alpha}-q\right\Vert $) is 0. If $q$
is new to $W$, the error will depend on whether the items stored
in $W$ are close to $q$. Usually, the higher the error, the more
novel $q$ is w.r.t $W$. 

\subsection{SM's Implementation Detail\label{subsec:Implementation-Details}}

In practice, the short-term memory $\mathcal{M}$ is a tensor of shape
$\left[B,N,d\right]$ where $B$ is the number of actors, $N$ the
memory length and $d$ the slot size. $B$ is the SG's hyperparameters
and tuned depending on tasks based on SG performance. For example,
for the Noisy-TV, we tune RND as the SG, obtaining $B=64$ and directly
using them for $\mathcal{M}$. $N$ and $d$ are the special hyperparameters
of our method. As mentioned in Sec. \ref{subsec:Ablation-Study},
we fix $N=128$ and $d=16$ in all experiments. As $B$ increases
in large-scale experiments, memory storage for $\mathcal{M}$ can
be demanding. To overcome this issue, we can use the uniform writing
trick to optimally preserve information while reducing $N$ \cite{le2019learning}. 

Also, for $\mathcal{W}$, by using a small hidden size, we can reduce
the requirement for physical memory significantly. Practically, in
all experiments, we implement $\mathcal{W}$ as a 2-layer feed-forward
neural network with a hidden size of 32 ($2n\rightarrow32\rightarrow2n$).
The activation is $\mathrm{tanh}$. With $n=512$ $d=16$, the number
of parameters of $\mathcal{W}$ is only about 65K. Also, $Q\in\mathbb{R}^{n\times d}$
and $V\in\mathbb{R}^{d\times n}$ have about 8K parameters. In total,
our SM only introduces less than 90K trainable parameters, which are
marginal to that of the SG and policy/value networks (up to 10 million
parameters).

\subsection{Intrinsic Reward Normalization\label{subsec:Intrinsic-Reward-Integration}}

Following \cite{burda2018exploration}, to make the intrinsic reward
on a consistent scale, we normalized the intrinsic reward by dividing
it by a running estimate of the standard deviations of the intrinsic
returns. This normalized intrinsic reward (NIR) will be used for training.
In addition, there is a hyperparameter named intrinsic reward coefficient
to scale the intrinsic contribution relative to the external reward.
We denote the running mean's standard deviations and intrinsic reward
coefficient as $r_{t}^{std}$ and $\beta$, respectively, in Algo.
\ref{alg:ir_full}. In our experiments, if otherwise stated, $\beta=1$. 

We note that when comparing the intrinsic reward at different states
in the same episode (as in the experiment section), we normalize intrinsic
rewards by subtracting the mean, followed by a division by the standard
deviation of all intrinsic rewards in the episode. Hence, the mean-normalized
intrinsic reward (MNIR) in these experiments is different from the
one used in training and can be negative. We argue that normalizing
with mean and std. of the episode's intrinsic rewards is necessary
to make the comparison reasonable. For example, in an episode, method
A assigns all steps with intrinsic rewards of 200; and method B assigns
novel steps with intrinsic rewards of 1 while others 0. Clearly, method
A treats all steps in the episode equally, and thus, it is equivalent
to giving no motivation for all of the steps in the episode (the learned
policy will not motivate the agent to visit novel states). On the
contrary, method B triggers motivation for novel steps in the episodes
(the learned policy will encourage visits to novel states). Without
normalizing by mean subtraction, it is tempting to conclude that the
relative intrinsic reward of method A for a novel step is higher,
which is technically incorrect.

\subsection{Experimental Details}

\subsubsection{Noisy-TV\label{subsec:Noisy-TV}}

We create the Noisy-TV environment by modifying the Maze environment
(MazeS3Fast-v0) in the MiniWorld library (Apache License) \cite{gym_miniworld}.
The backbone RL algorithm is PPO. We adopt a public code repository
for the implementation of PPO and RND (MIT License)\footnote{\url{https://github.com/jcwleo/random-network-distillation-pytorch}}.
In this environment, the state is an image of the agent's viewport.
The details of architecture and hyperparameters of the backbone and
RND is presented in Table \ref{tab:atari_RND}. Most of the setting
is the same as in the repository. We only tune the number of actors
(32, 128, 1024), mini-batch size (4, 16, 64) and $\epsilon$-clip
(0.1, 0.2, 0.3) to suit our hardware and the task. After tuning with
RND, we use the same setting for our RND+SM. 

Fig. \ref{fig:ntv_all} reports all results for this environment.
Fig. \ref{fig:ntv_all} (a) compares the final intrinsic reward (IR)
generated by RND and RND+SM over training time. Overall, RND's IR
is always higher than RND+SM's, indicating that our method significantly reduces
 the attention of the agent to the noisy TV by assigning less
IR to watching TV. Fig. \ref{fig:ntv_all} (b) compares the number
of noisy actions between two methods where RND+SM consistently shows
fewer watching TV actions. That confirms RND+SM agent is less distracted
by the TV. 

As mentioned in the main text, RND+SM is better at handling noise
than RND. Note that RND aims to predict the transformed states by
minimizing $\left\Vert SG\left(s_{t}\right)-f_{R}(s_{t})\right\Vert $
where $f_{R}$ is a fixed neural network initialized randomly. If
RND can learn the transformation, it can pass through the state,
which is similar to reconstruction in an autoencoder. However, learning
$f_{R}$ can be harder and require more samples than learning an identity
transformation since $f_{R}$ is non-linear and complicated. Hence,
it may be more challenging for RND to pass through the noise than
SM.

Another possible reason lies in the operating space (state vs. surprise).
If we treat white noise as a random variable $X$, a surprise generator
(SG) can at most learn to predict the mean of this variable and compute
the surprise $U=\mathbb{E}\left[X|Y\right]-X$ where $Y$ is a random
factor that affects the training of the surprise generator. The factor
$Y$ makes the SG produce imperfect reconstruction $\mathbb{E}\left[X|Y\right]$\footnote{In this case, the perfect reconstruction is $\mathbb{E}\left[X\right]$}.
Here, SG and SM learn to reconstruct $X$ and $U$, respectively.
We can prove that the variance of each feature dimension in $U$ is
smaller than that of $X$ (see Sec. \ref{subsec:Variance-Inequality}).
Learning an autoencoder on surprise space is more beneficial than
in state space since the data has less variance and thus, it may require
less data points to learn the data distribution.

Fig. \ref{fig:ntv_all} (c) reports the performance of all baselines.
Besides RND and RND+SM, we also include PPO without intrinsic reward
as the vanilla Baseline for reference. In addition, we investigate
a simple implementation of SM using a count-based method to measure
surprise novelty. Concretely, we use SimHash algorithm to count the
number of surprises $c(u_{t})$ in a similar manner as \cite{bellemare2016unifying}
and name the baseline RND+SM (count). The intrinsic reward is then
$\beta/\sqrt{c(u_{t})}$. We tune the hyperparameter $\beta=\left\{ 0.5,1,5\right\} $
and the hash matrix size $k_{h}=\left\{ 32,64,128,256\right\} $ and
use the same normalization and training process to run this baseline.
We report the learning curves of the best variant with $\beta=0.5$
and $k_{h}=128$. For Disagreement model, we use the author codebase\footnote{\url{https://github.com/pathak22/exploration-by-disagreement/tree/master}}, and tune the number of dynamic models $\in \{3,5,7,9\}$.   

The result demonstrates that the proposed SM using
memory-augmented neural networks outperforms the count-based SM by
a significant margin. One possible reason is that the count-based method
cannot handle white noise: it always returns high intrinsic rewards.
In contrast, our SM can somehow reconstruct white noise via pass-through
mechanism and thus reduces the impact of fake surprises on learning.
Also, the proposed SM is more flexible than the count-based counterpart
since it learns to reconstruct from the data rather than using a fixed
counting scheme. The result also shows that RND+SM outperforms the
vanilla Baseline. Although the improvement is moderate (0.9 vs 0.85),
the result is remarkable since the Noisy-TV is designed to fool intrinsic
motivation methods and among all, only RND+SM can outperform the vanilla
Baseline.

\begin{figure*}
\begin{centering}
\includegraphics[width=1\linewidth]{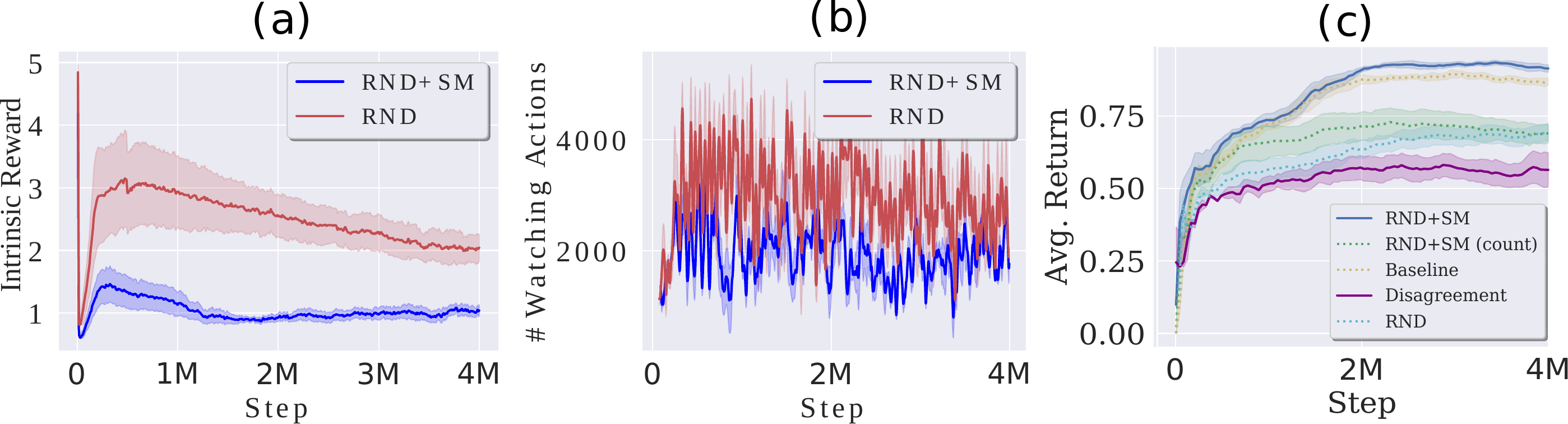}
\par\end{centering}
\caption{Noisy-TV's learning curves over training steps: (a) Average NIR. (b)
Number of watching actions. (c) Average return. All curves are averaged
over 5 runs (mean$\pm$std). \label{fig:ntv_all}}
\end{figure*}

\subsubsection{MiniGrid\label{subsec:MiniGrida}}

The tasks in this experiment are from the MiniGrid library (Apache
License) \cite{gym_minigrid}. In MiniGrid environments, the state
is a description vector representing partial observation information
such as the location of the agents, objects, moving directions, etc.
The three tasks use hardest maps:
\begin{itemize}
\item DoorKey: MiniGrid-DoorKey-16x16-v0 
\item LavaCrossing: MiniGrid-LavaCrossingS11N5-v0 
\item DynamicObstacles: MiniGrid-Dynamic-Obstacles-16x16-v0 
\end{itemize}
The SGs used in this experiment are RND \cite{burda2018exploration},
ICM \cite{pathak2017curiosity}, NGU \cite{badia2019never} and AE.
Below we describe the input-output structure of these SGs.
\begin{itemize}
\item RND: $I_{t}=s_{t}$ and $O_{t}=f_{R}\left(s_{t}\right)$ where $s_{t}$
is the current state and $f_{R}$ is a neural network that has a similar
structure as the prediction network, yet its parameters are initialized
randomly and fixed during training.
\item ICM: $I_{t}=\left(s_{t-1},a_{t}\right)$ and $O_{t}=s_{t}$ where
$s$ is the embedding of the state and $a$ the action. We note that
in addition to the surprise loss (Eq. \ref{eq:surlss}), ICM is trained
with inverse dynamics loss. 
\item NGU: This agent reuses the RND as the SG ($I_{t}=s_{t}$ and $O_{t}=f_{R}\left(s_{t}\right)$)
and combines the surprise norm with an KNN episodic reward. When applying
our SM to NGU, we only take the surprise-based reward as input to
the SM. The code for NGU is based on this public repository \url{https://github.com/opendilab/DI-engine}.
\item AE: $I_{t}=s_{t}$ and $O_{t}=s_{t}$ where $s$ is the embedding
of the state. This SG can be viewed as an associative memory of the
observation, aiming to remember the states. This baseline is designed
to verify the importance of surprise modeling. Despite sharing a similar
architecture, it differs from our SM, which operates on surprise and
have an augmented episodic memory to support reconstruction.
\end{itemize}
The backbone RL algorithm is PPO. The code for PPO and RND is the
same as in Sec. \ref{subsec:Noisy-TV}. We adopt a public code repository
for the implementation of ICM (MIT License)\footnote{\url{https://github.com/jcwleo/curiosity-driven-exploration-pytorch}}.
We implement AE ourselves using a 3-layer feed-forward neural network.
For the SGs, we only tune the number of actors (32, 128, 1024), mini-batch
size (4, 16, 64) and $\epsilon$-clip (0.1, 0.2, 0.3) for the DoorKey
task. We also tune the architecture of the AE (number of layers: 1,2
or 3, activation $\mathrm{tanh}$ or $\mathrm{ReLU}$) on the DoorKey
task. After tuning the SGs, we use the same setting for our SG+SM.
The detailed configurations of the SGs for this experiment are reported
in Table \ref{tab:icm_ae} and Table \ref{tab:atari_RND}. 

The full learning curves of the backbone (Baseline), SG and SG+SM
are given in Fig. \ref{fig:mng_all-1}. To visualize the difference
between surprise and surprise residual vectors, we map these in the
synthetic trajectory to 2-dimensional space using t-SNE projection
in Fig. \ref{tnse-1}. The surprise points show clustered patterns
for high-MNIR states, which confirms our hypothesis that there exist
familiar surprises (they are highly surprising due to high norm, yet
repeated). In contrast, the surprise residual estimated by the SM
has no high-MNIR clusters. The SM transforms clustered surprises to
scatter surprise residuals, resulting in a broader range of MNIR,
thus showing significant discrimination on states that have similar
surprise norm. 

\begin{table}
\begin{centering}
{\footnotesize{}}%
\begin{tabular}{ccc}
\hline 
\multicolumn{1}{c}{{\small{}Hyperparameters}} & {\small{}ICM} & {\small{}AE}\tabularnewline
\hline 
\multicolumn{1}{c}{{\small{}PPO's state}} & {\small{}3-layer feedforward} & {\small{}3-layer feedforward}\tabularnewline
{\small{}encoder} & {\small{}net (Tanh, h=256)} & {\small{}net (Tanh, h=256)}\tabularnewline
\multicolumn{1}{c}{{\small{}SG's surprise}} & {\small{}4-layer feedforward} & {\small{}3-layer feedforward}\tabularnewline
{\small{}predictor} & {\small{}net (ReLU, h=512)} & {\small{}net (Tanh, h=512)}\tabularnewline
Intrisic Coef. $\beta$ & 1 & 1\tabularnewline
{\small{}Num. Actor $B$} & {\small{}64} & {\small{}64}\tabularnewline
{\small{}Minibatch size} & {\small{}64} & {\small{}64}\tabularnewline
{\small{}Horizon $T$} & {\small{}128} & {\small{}128}\tabularnewline
{\small{}Adam Optimizer's lr} & {\small{}$10^{-4}$} & {\small{}$10^{-4}$}\tabularnewline
{\small{}Discount $\gamma$} & {\small{}0.999} & {\small{}0.999}\tabularnewline
{\small{}Intrinsic $\gamma^{i}$} & {\small{}0.99} & {\small{}0.99}\tabularnewline
{\small{}GAE $\lambda$} & {\small{}0.95} & {\small{}0.95}\tabularnewline
{\small{}PPO's clip $\epsilon$} & {\small{}0.2} & {\small{}0.2}\tabularnewline
\hline 
\end{tabular}{\footnotesize\par}
\par\end{centering}
\caption{Hyperparameters of ICM and AE (PPO backbone). \label{tab:icm_ae}}
\end{table}
\begin{figure*}
\begin{centering}
\includegraphics[width=1\linewidth]{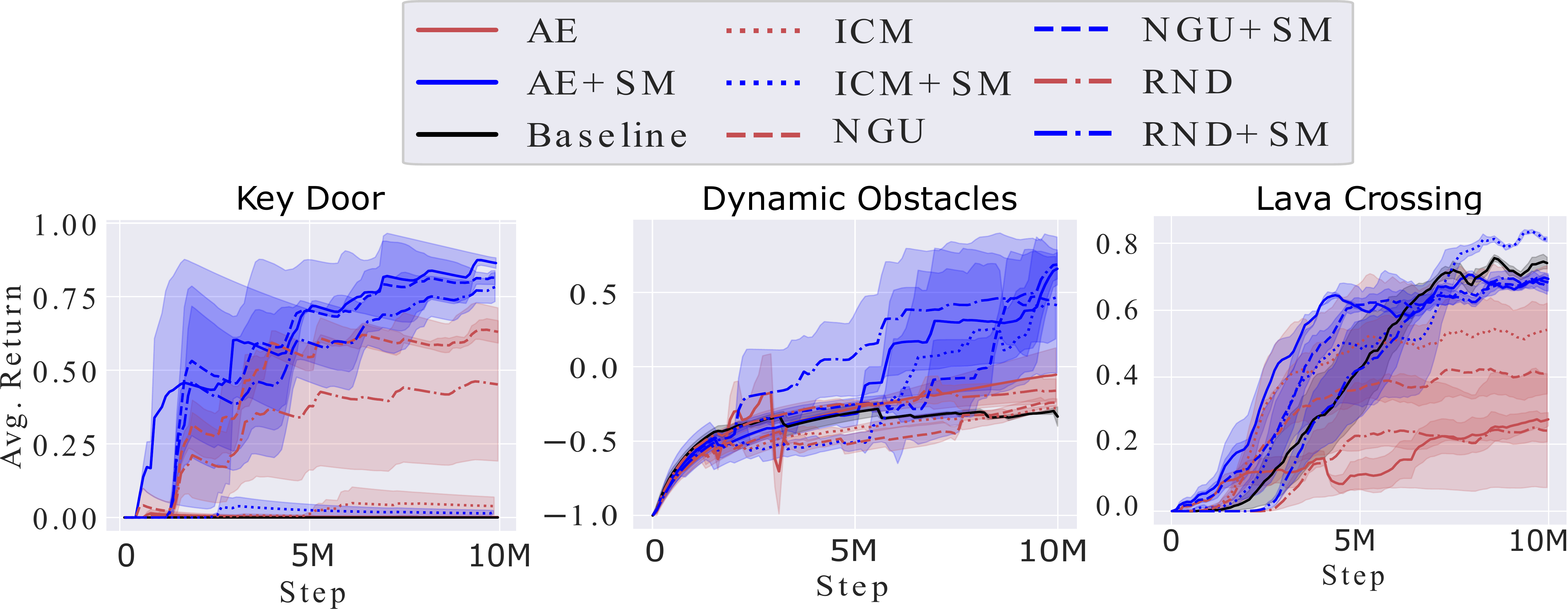}
\par\end{centering}
\caption{Minigrid's learning curves over 10 million training steps (mean$\pm$std.
over 5 runs). \label{fig:mng_all-1}}
\end{figure*}
\begin{figure*}
\begin{centering}
\includegraphics[width=1\linewidth]{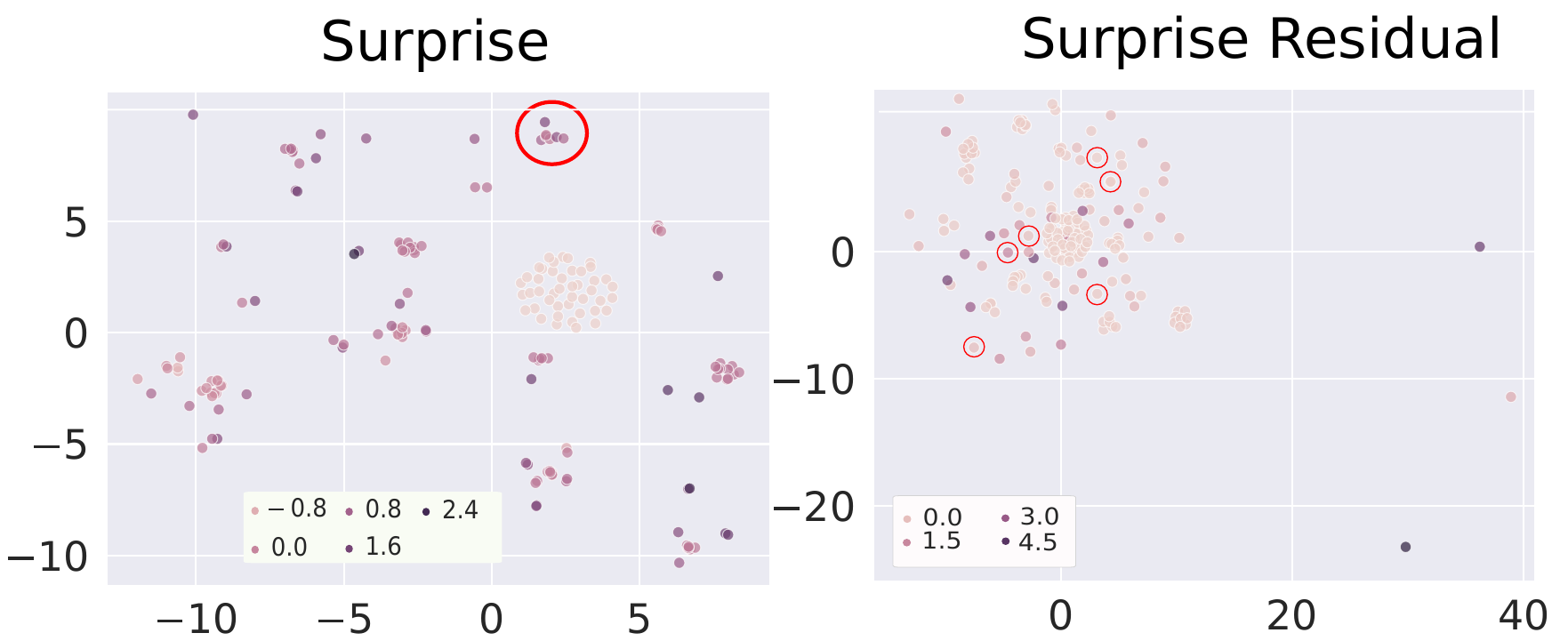}
\par\end{centering}
\caption{Key-Door: t-NSE 2d representations of surprise ($u_{t}$) and surprise
residual ($\hat{q}_{t}-q_{t}$). Each point corresponds to the MNIR
at some step in the episode. Color denotes the MNIR value (darker
means higher MNIR). The red circle on the left picture shows an example
cluster of 6 surprise points. Surprise residuals of these points are
not clustered, as shown in 6 red circles on the right pictures. In
other words, surprise residual can discriminate surprises with similar
norms. \label{tnse-1}}
\end{figure*}

\subsubsection{Atari\label{subsec:aAtari}}

The Atari 2600 Games task involves training an agent to achieve high
game scores. The state is a 2d image representing the screen of the
game.

\begin{table*}
\begin{centering}
{\footnotesize{}}%
\begin{tabular}{cccc}
\hline 
\multicolumn{1}{c}{{\footnotesize{}Hyperparameters}} & {\footnotesize{}MiniGrid} & {\footnotesize{}Noisy-TV+MiniWorld} & {\footnotesize{}Atari}\tabularnewline
\hline 
\multicolumn{1}{c}{{\footnotesize{}PPO's state}} & {\scriptsize{}3-layer feedforward} & {\scriptsize{}3-layer Leaky-ReLU CNN with} & {\scriptsize{}3-layer Leaky-ReLU CNN with}\tabularnewline
{\footnotesize{}encoder} & {\scriptsize{}net (Tanh, h=256)} & {\scriptsize{}kernels } & {\scriptsize{}kernels }\tabularnewline
 &  & {\scriptsize{}$\left\{ 12/32/8/4,32/64/4/2,64/64/3/1\right\} $} & {\scriptsize{}$\left\{ 4/32/8/4,32/64/4/2,64/64/3/1\right\} $}\tabularnewline
 &  & {\scriptsize{}+2-layer feedforward net} & {\scriptsize{}+2-layer feedforward net}\tabularnewline
 &  & {\scriptsize{}(ReLU, h=256)} & {\scriptsize{}(ReLU, h=256)}\tabularnewline
\multicolumn{1}{c}{{\footnotesize{}RND's surprise}} & {\scriptsize{}3-layer feedforward} & {\scriptsize{}3-layer Leaky-ReLU CNN with} & {\scriptsize{}3-layer Leaky-ReLU CNN with}\tabularnewline
{\footnotesize{}predictor} & {\scriptsize{}net (Tanh, h=512)} & {\scriptsize{}kernels } & {\scriptsize{}kernels}\tabularnewline
 &  & {\scriptsize{}$\left\{ 1/32/8/4,32/64/4/2,64/64/3/1\right\} $} & {\scriptsize{}$\left\{ 1/32/8/4,32/64/4/2,64/64/3/1\right\} $}\tabularnewline
 &  & {\scriptsize{}+2-layer feedforward net} & {\scriptsize{}+2-layer feedforward net }\tabularnewline
 &  & {\scriptsize{}(ReLU, h=512)} & {\scriptsize{}(ReLU, h=512)}\tabularnewline
{\footnotesize{}Intrinsic Coef. $\beta$} & {\footnotesize{}1} & {\footnotesize{}1} & {\footnotesize{}1}\tabularnewline
{\footnotesize{}Num. Actor $B$} & {\footnotesize{}64} & {\footnotesize{}64} & {\footnotesize{}128}\tabularnewline
{\footnotesize{}Minibatch size} & {\footnotesize{}64} & {\footnotesize{}64} & {\footnotesize{}4}\tabularnewline
{\footnotesize{}Horizon $T$} & {\footnotesize{}128} & {\footnotesize{}128} & {\footnotesize{}128}\tabularnewline
{\footnotesize{}Adam Optimizer's lr} & {\footnotesize{}$10^{-4}$} & {\footnotesize{}$10^{-4}$} & {\footnotesize{}$10^{-4}$}\tabularnewline
{\footnotesize{}Discount $\gamma$} & {\footnotesize{}0.999} & {\footnotesize{}0.999} & {\footnotesize{}0.999}\tabularnewline
{\footnotesize{}Intrinsic $\gamma^{i}$} & {\footnotesize{}0.99} & {\footnotesize{}0.99} & {\footnotesize{}0.99}\tabularnewline
{\footnotesize{}GAE $\lambda$} & {\footnotesize{}0.95} & {\footnotesize{}0.95} & {\footnotesize{}0.95}\tabularnewline
{\footnotesize{}PPO's clip $\epsilon$} & {\footnotesize{}0.2} & {\footnotesize{}0.1} & {\footnotesize{}0.1}\tabularnewline
\hline 
\end{tabular}{\footnotesize\par}
\par\end{centering}
\caption{Hyperparameters of RND (PPO backbone). \label{tab:atari_RND}}

\end{table*}

\paragraph{SG and RL backbone implementations}

We use 2 SGs: RND and LWM. RND uses a PPO backbone as in previous
sections. On the other hand, LWM uses DQN backbone with CNN-based
encoder and GRU-based value function. The LWM SG uses GRU to model
forward dynamics of the environment and thus its input is: $I_{t}=\left(s_{t-1},a_{t},h_{t-1}\right)$
where $s_{t-1}$ is the embedding of the previous state, $a_{t}$
the current action, and $h_{t-1}$ the hidden state of the world model
GRU. The target $O_{t}$ is the embedding of the current state $s_{t}$. 

RND follows the same implementation as in previous experiments. We
use the public code of LWM provided by the authors\footnote{\url{https://github.com/htdt/lwm}}
to implement LWM. The hyperparameters of RND and LWM are tuned by
the repository's owner (see Table \ref{tab:atari_RND} for RND and
refer to the code or the original paper \cite{ermolov2020latent}
for the details of LWM implementation). We augment them with our SM
of default hyperparameters $N=128,d=16$.

\paragraph*{Training and evaluation}

We follow the standard training for Atari games, such as stacking
four frames and enabling sticky actions. All the environments are
based on OpenAI's gym-atari's NoFrameskip-v4 variants (MIT Liscence)\footnote{\url{https://github.com/openai/gym}}
. After training, we evaluate the models by measuring the average
return over 128 episodes and report the results in Table. \ref{tab:atari6}.
Depending on the setting, the models are trained for 50 or 200 million
frames. 

\paragraph{Results }

Fig. \ref{fig:atari50} demonstrates the learning curves of all models
in 6 Atari games under the low-sample regime. LWM+SM and RND+SM clearly
outperfrom LWM and RND in Frostbite, Venture, Gravitar, Solaris and
Frostbite, Venture, Gravitar and MontezumaRevenge, respectively. Table
\ref{tab:atari6-1} reports the results of more baselines.

\begin{figure*}
\begin{centering}
\includegraphics[width=1\linewidth]{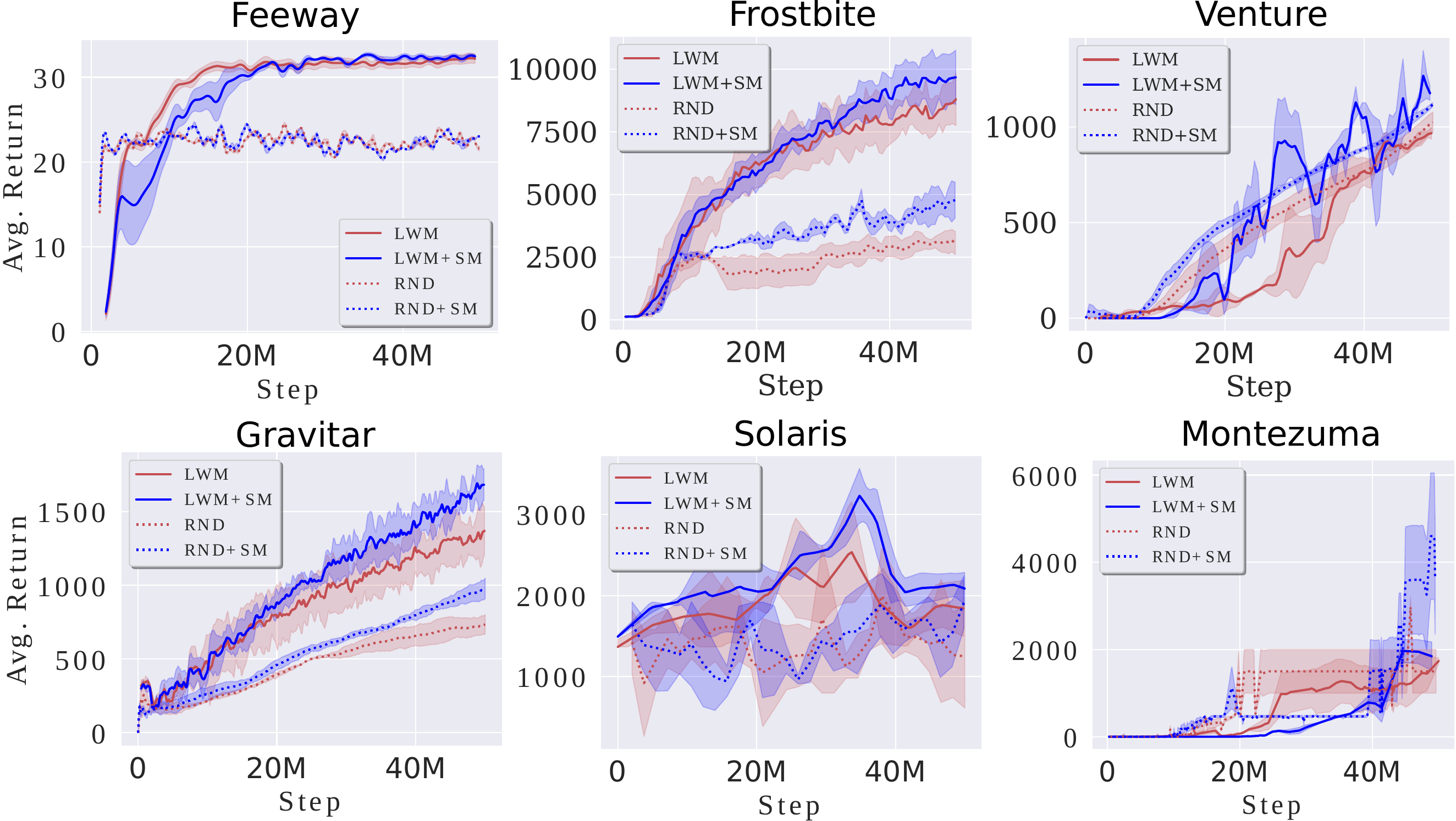}
\par\end{centering}
\caption{Atari low-sample regime: learning curves over 50 million frames (mean$\pm$std.
over 5 runs). To aid visualization, we smooth the curves by taking
average  over a window sized 50. \label{fig:atari50}}
\end{figure*}
\begin{table*}
\begin{centering}
{\footnotesize{}}%
\begin{tabular}{ccccccc|cccc}
\hline 
\multirow{1}{*}{{\scriptsize{}Task}} & {\scriptsize{}EMI$^{\spadesuit}$} & {\scriptsize{}EX2$^{\spadesuit}$} & {\scriptsize{}ICM$^{\spadesuit}$} & {\scriptsize{}AE-SH$^{\spadesuit}$} & \multirow{1}{*}{{\scriptsize{}LWM$^{\spadesuit}$}} & {\scriptsize{}RND$^{\spadesuit}$} & {\scriptsize{}LWM$^{\diamondsuit}$} & {\scriptsize{}LWM+SM$^{\diamondsuit}$} & {\scriptsize{}RND$^{\diamondsuit}$} & {\scriptsize{}RND+SM$^{\diamondsuit}$}\tabularnewline
\hline 
{\scriptsize{}Freeway} & \textbf{\scriptsize{}33.8} & {\scriptsize{}27.1} & {\scriptsize{}33.6} & {\scriptsize{}33.5} & {\scriptsize{}30.8} & {\scriptsize{}33.3} & {\scriptsize{}31.1} & {\scriptsize{}31.6} & {\scriptsize{}22.2} & {\scriptsize{}22.2}\tabularnewline
{\scriptsize{}Frostbite} & {\scriptsize{}7,002} & {\scriptsize{}3,387} & {\scriptsize{}4,465} & {\scriptsize{}5,214} & {\scriptsize{}8,409} & {\scriptsize{}2,227} & {\scriptsize{}8,598} & \textbf{\emph{\scriptsize{}10,258}} & {\scriptsize{}2,628} & \emph{\scriptsize{}5,073}\tabularnewline
{\scriptsize{}Venture} & {\scriptsize{}646} & {\scriptsize{}589} & {\scriptsize{}418} & {\scriptsize{}445} & {\scriptsize{}998} & {\scriptsize{}707} & {\scriptsize{}985} & \textbf{\emph{\scriptsize{}1,381}} & {\scriptsize{}1,081} & \emph{\scriptsize{}1,119}\tabularnewline
{\scriptsize{}Gravitar} & {\scriptsize{}558} & {\scriptsize{}550} & {\scriptsize{}424} & {\scriptsize{}482} & {\scriptsize{}1,376} & {\scriptsize{}546} & {\scriptsize{}1,242} & \textbf{\emph{\scriptsize{}1,693}} & {\scriptsize{}739} & \emph{\scriptsize{}987}\tabularnewline
{\scriptsize{}Solaris} & {\scriptsize{}2,688} & {\scriptsize{}2,276} & {\scriptsize{}2,453} & \textbf{\scriptsize{}4,467} & {\scriptsize{}1,268} & {\scriptsize{}2,051} & {\scriptsize{}1,839} & \emph{\scriptsize{}2,065} & {\scriptsize{}2,206} & \emph{\scriptsize{}2,420}\tabularnewline
{\scriptsize{}Montezuma} & {\scriptsize{}387} & {\scriptsize{}0} & {\scriptsize{}161} & {\scriptsize{}75} & {\scriptsize{}2,276} & {\scriptsize{}377} & {\scriptsize{}2,192} & {\scriptsize{}2,269} & {\scriptsize{}2,475} & \textbf{\emph{\scriptsize{}5,187}}\tabularnewline
\hline 
{\scriptsize{}Norm. Mean} & {\scriptsize{}61.4} & {\scriptsize{}40.5} & {\scriptsize{}46.1} & {\scriptsize{}52.4} & {\scriptsize{}80.6} & {\scriptsize{}42.2} & {\scriptsize{}80.5} & \textbf{\emph{\scriptsize{}97.0}} & {\scriptsize{}50.7} & \emph{\scriptsize{}74.8}\tabularnewline
{\scriptsize{}Norm. Median} & {\scriptsize{}34.9} & {\scriptsize{}32.3} & {\scriptsize{}23.1} & {\scriptsize{}33.3} & {\scriptsize{}60.8} & {\scriptsize{}32.7} & {\scriptsize{}66.5} & \emph{\scriptsize{}83.7} & {\scriptsize{}58.3} & \textbf{\emph{\scriptsize{}84.6}}\tabularnewline
\hline 
\end{tabular}{\footnotesize\par}
\par\end{centering}
~

\caption{Atari: test performance after 50 million training frames (mean over
5 runs). $\spadesuit$ is from a prior work \cite{ermolov2020latent}.
$\diamondsuit$ is our run. The last two rows are mean and median
human normalized scores. Bold denotes best results. Italic denotes
that SG+SM is significantly better than SG regarding Cohen effect
size less than 0.5. \label{tab:atari6-1}}
\end{table*}

\subsubsection{Ablation study\label{subsec:aAblation-study}}

\paragraph{Role of Memories}

We conduct more ablation studies to verify the need for the short
$\mathcal{M}$ and long-term ($\mathcal{W}$) memory in our SM. We
design additional baselines SM (no $\mathcal{W}$) and SM (no $\mathcal{M}$)
(see Sec. \ref{subsec:Ablation-Study}), and compare them with the
SM with full features in Montezuma Revenge and Frostbite task. Fig.
\ref{fig:mr_abl} (a) shows that only SM (full) can reach an average
score of more than 5000 after 50 million training frames. Other ablated
baselines can only achieve around 2000 scores. 

We also shows the impact of the episodic memory in decreasing the
intrinsic rewards for similar states as discussed in Sec. \ref{subsec:Surprise-Memory}.
We select 3 states in the MiniGrid's KeyDoor task and computes the
MNIR for each state, visualized in Fig. \ref{fig:mr_abl-1}. At the
step-1 state, the MNIR is low since there is nothing special in the
view of the agent. At the step-15 state, the agent first sees the
key, and get a high MNIR. At the step-28 state, the agent drops the
key and sees the key again. This event is still more interesting than
the step-1 state. However, the view is similar to the one in step
15, and thus, the MNIR decreases from 0.7 to 0.35 as expected. 

\begin{figure*}
\begin{centering}
\includegraphics[width=1\linewidth]{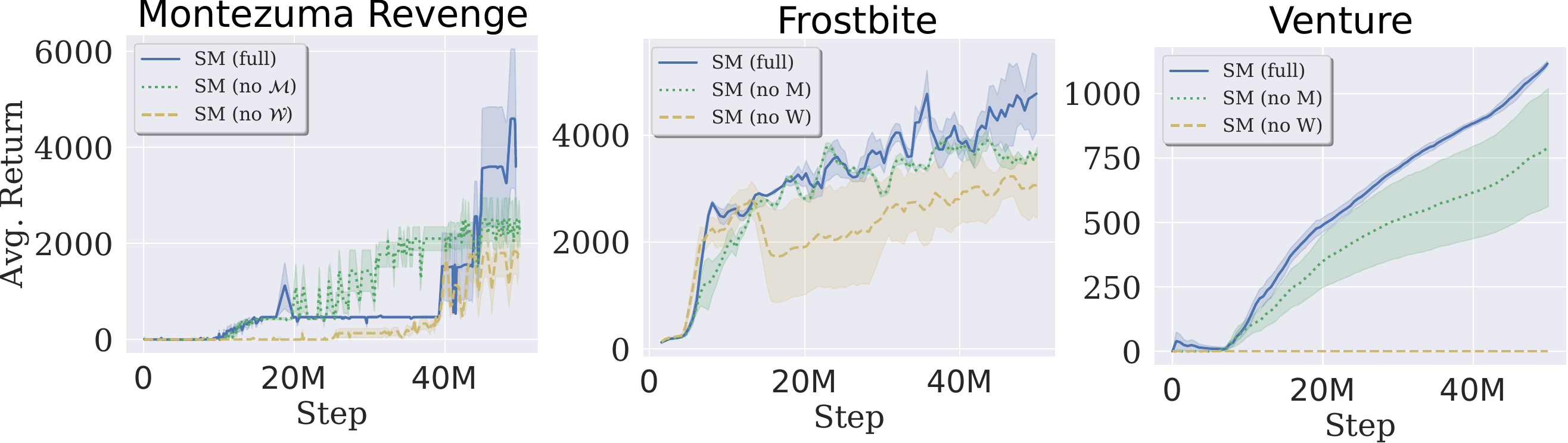}
\par\end{centering}
\caption{Ablation study: average returns (mean$\pm$std.) over 5 runs.\label{fig:mr_abl}}
\end{figure*}
\begin{figure*}
\begin{centering}
\includegraphics[width=1\linewidth]{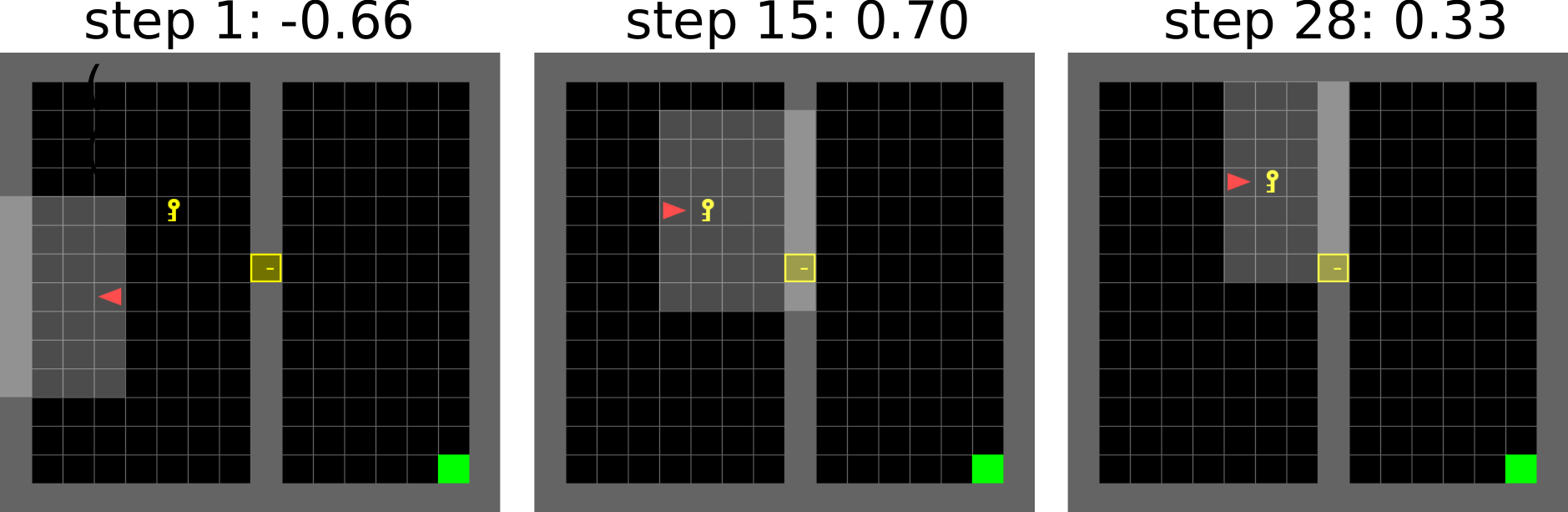}
\par\end{centering}
\caption{MiniGrid's KeyDoor: MNIR of SM at different steps in an episode. \label{fig:mr_abl-1}}
\end{figure*}

\paragraph{No Task Reward\label{subsec:MiniWorlda}}

The tasks in this experiment are from the MiniWorld library (Apache
License) \cite{gym_miniworld}. The two tasks are:
\begin{itemize}
\item Easy: MiniWorld-PickupObjs-v0 
\item Hard: MiniWorld-FourRooms-v0 
\end{itemize}
The backbone and SG are the same as in Sec. \ref{subsec:Noisy-TV}.
We remove the task/external reward in this experiment. For the Baseline,
without task reward, it receives no training signal and thus, showing
a similar behavior as a random agent. Fig. \ref{fig:norall} illustrates
the running average of cumulative task return and the intrinsic reward
over training steps. In the Easy mode, the random Baseline can even
perform better than RND, which indicates that biased intrinsic reward
is not always helpful. RND+SM, in both modes, shows superior performance,
confirming that its intrinsic reward is better to guide the exploration
than that of RND. 

\begin{figure*}
\begin{centering}
\includegraphics[width=1\linewidth]{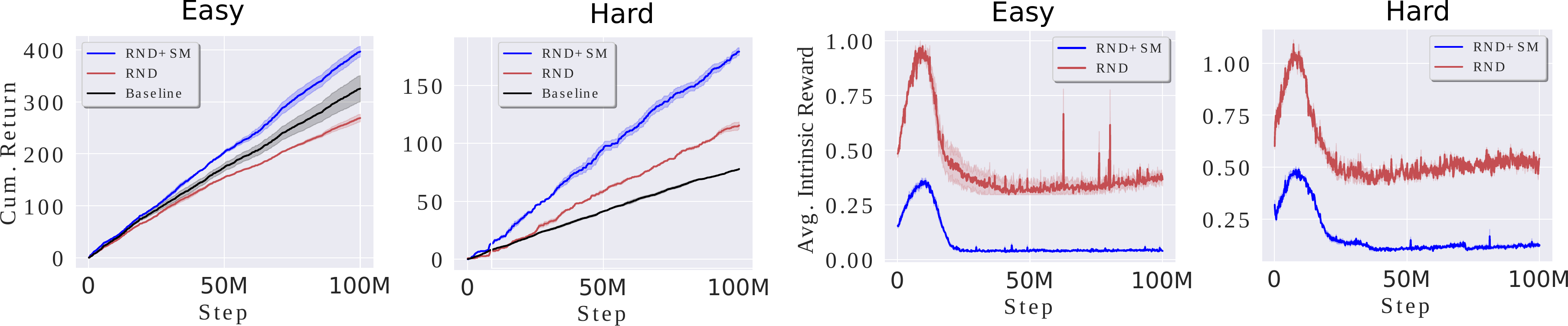}
\par\end{centering}
\caption{MiniWorld: Exploration without task reward. Left: Cumulative task
returns over 100 million training steps for two setting: Easy (1 room,
3 objects) and Hard (4 rooms, 1 object). Right: The average intrinsic
reward over training time. The learning curves are taken average (mean$\pm$std.)
over 5 runs \label{fig:norall}}
\end{figure*}

\subsection{Theoretical Property of Surprise Space's Variance\label{subsec:Variance-Inequality}}

Let $X$ be a random variable representing the observation at some
timestep, a surprise generator (SG) can at most learn to predict the
mean of this variable and compute the surprise $U=\mathbb{E}\left[X|Y\right]-X$
where $Y$ is a random factor that affect the prediction of SG and
makes it produce imperfect reconstruction $\mathbb{E}\left[X|Y\right]$
instead of $\mathbb{E}\left[X\right]$. For instance, in the case
of an autoencoder AE as the SG, $X$ and $U$ are $s_{t}$and $AE(s_{t})-s_{t}$,
respectively.

Let us denote $Z=\mathbb{E}\left(X|Y\right)$, then $\mathbb{E}\left[Z|Y\right]=Z$
and $\mathbb{E}\left[Z^{2}|Y\right]=Z^{2}$. We have 

\begin{align*}
\mathrm{{var}}\left(X\right) & =\mathrm{{var}}(X-Z+Z)\\
 & =\mathrm{{var}}(X-Z)+\mathrm{{var}}(Z)+2\mathrm{{cov}}(X-Z,Z)\\
 & =\mathrm{{var}}(X-Z)+\mathrm{{var}}(Z)+2\mathbb{E}[(X-Z)Z]\\
 & -2\mathbb{E}[X-Z]\mathbb{E}[Z]
\end{align*}
Using the Law of Iterated Expectations, we have

\begin{align*}
\mathbb{E}[X-Z] & =\mathbb{E}[\mathbb{E}[X-Z|Y]]\\
 & =\mathbb{E}[\mathbb{E}\left[X|Y\right]-\mathbb{E}\left[Z|Y\right]]\\
 & =\mathbb{E}\left[Z-Z\right]=0
\end{align*}
and 

\begin{align*}
\mathbb{E}[(X-Z)Z] & =\mathbb{E}[\mathbb{E}[(X-Z)Z|Y]]\\
 & =\mathbb{E}[\mathbb{E}[XZ-Z^{2}|Y]]\\
 & =\mathbb{E}[\mathbb{E}\left(XZ|Y\right)-\mathbb{E}\left(Z^{2}|Y\right)]\\
 & =\mathbb{E}[Z\mathbb{E}\left(X|Y\right)-Z^{2}]\\
 & =\mathbb{E}[Z^{2}-Z^{2}]=0
\end{align*}
Therefore,

\[
\mathrm{{var}}\left(X\right)=\mathrm{{var}}(X-Z)+\mathrm{{var}}(Z)
\]
Let $C_{ii}^{X}$, $C_{ii}^{X-Z}$ and $C_{ii}^{Z}$ denote the diagonal
entries of these covariance matrices, they are the variances of the
components of the random vector $X$, $X-Z$ and $Z$, respectively.
That is, 

\begin{align*}
\left(\sigma_{i}^{X}\right)^{2} & =\left(\sigma_{i}^{X-Z}\right)^{2}+\left(\sigma_{i}^{Z}\right)^{2}\\
\Rightarrow\left(\sigma_{i}^{X}\right)^{2} & \geq\left(\sigma_{i}^{X-Z}\right)^{2}=\left(\sigma_{i}^{U}\right)^{2}
\end{align*}
In our setting, $X$ and $U$ represents observation and surprise
spaces, respectively. Therefore, the variance of each feature dimension
in surprise space is smaller than that of observation space. The equality
is obtained when $\left(\sigma_{i}^{Z}\right)^{2}=0$ or $\mathbb{E}\left(X|Y\right)=\mathbb{E}\left(X\right)$.
That is, the SG's prediction is perfect, which is unlikely to happen
in practice. 

\subsection{Further Discussions\label{subsec:Further-Discussions}}

If we view surprise as the first-order error
between the observation and the predicted, surprise novelty--the
retrieval error between the surprise and the reconstructed memory,
is essentially the second-order error. It would be interesting to
investigate the notion of higher-order errors, study their theoretical
properties, and utilize them for intrinsic motivation in our future
work.

One limitation of our method is that
$\mathcal{M}$ and $\mathcal{W}$ require additional physical memory
(RAM/GPU) than SG methods. Also, a plug-in module like SM introduces
more hyperparameters, such as N and d. Although we find the default
values of N=128 and d=16 work well across all experiments in this
paper, we recommend adjustments if users apply our method to novel
domains. Further discussion on the rationality of our architecture
choice is given in Appendix \ref{subsec:Further-Discussions}.

Our method assumes that surprises have patterns and can be remembered
by our surprise memory. There might exist environments beyond those
studied in this paper where this assumption may not hold, or surprise-based
counterparts already achieve optimal exploration (e.g., perfect SG)
and thus do not need SM for improvement (e.g., Freeway game). We acknowledge that our method can be less effective in case experiencing the same type of surprise has great value, which can happen when the surprise representations are not well defined such that different events share similar surprise representations. This can be avoided by employing strong surprise generators that can learn to generate meaningful surprises.  Our objective is designed to address known shortcomings of maximizing surprise, particularly in environments with stochastic dynamics or partial observability. We believe that our approach offers a different perspective and complements existing exploration methods.

We note that there should be a balance point for the reconstructing
ability of the AE as long as it keeps minimizing the objective in
Eq. 5. This is critical to help our model overfit to reconstructing
rare events. For example, if we use other memory architectures to
store surprise novelty, as we train the agent the AE will get better
to recover that red box surprise at the end, contrarily, the AE probably
cannot learn so well to reconstruct the TV's noise. Thus, it fails
to give high rewards to the rare event (red box). It may not happen
with AE. To be specific, as the agent sees the red box more frequently,
it is true that AE tends to yield a relatively lower intrinsic reward
(IR) for the event, leading to nonoptimal policy. But it only lasts
for a short time. Since the nonoptimal policy will make the agent
visits non-goal states such as watching TV, the AE will be busy learning
to reconstruct these unimportant events. Due to limited capacity,
the AE will soon forget how to recall the red box and get better at
reconstructing other things like walls or passing through noises.
Hence, once again, the IR for seeing the red box will be high, aligning
the agent to the optimal policy.

\end{document}